\documentclass[10pt,journal,compsoc]{IEEEtran}
\newcommand{\myparagraph}[1]{\vspace{0.3\baselineskip}\noindent{\textbf{#1.}}~}
\usepackage[linesnumbered,ruled,vlined,boxed,noend]{algorithm2e}
\usepackage{booktabs}
\usepackage{amsmath}
\usepackage{amsthm}
\usepackage{amssymb}
\usepackage{graphicx}
\usepackage{enumitem}
\usepackage{multirow}

\setlist[itemize]{leftmargin=*}

\usepackage[colorlinks=true,linkcolor=red,urlcolor=black,citecolor=blue]{hyperref}

\DeclareMathOperator{\tm}{\textup{Tm}} 
\DeclareMathOperator{\m}{\textup{m}} 
\newcommand{\bdelta}{\boldsymbol{\delta}}

\newcommand{\bx}{\boldsymbol{x}}
\newcommand{\bc}{\boldsymbol{c}}

\newcommand{\var}{\mathrm{Var}}

\newcommand{\jkl}{\texttt{JKL}\xspace}
\newcommand{\mv}{\texttt{MV}\xspace}
\newcommand{\fr}{\texttt{FR}\xspace}
\newcommand{\nf}{\texttt{NF}\xspace}
\newcommand{\tkm}{\texttt{TKM}\xspace}
\newcommand{\sfr}{\texttt{SFR}\xspace}
\newcommand{\lloyd}{\texttt{Lloyd's heuristic}\xspace}
\newcommand{\kplus}{\texttt{$k$-means++}\xspace}
\newcommand{\athlete}{\textsf{Athlete}\xspace}
\newcommand{\bank}{\textsf{Bank}\xspace}
\newcommand{\census}{\textsf{Census}\xspace}
\newcommand{\diabetes}{\textsf{Diabetes}\xspace}
\newcommand{\recruitment}{\textsf{Recruitment}\xspace}
\newcommand{\spanish}{\textsf{Spanish}\xspace}
\newcommand{\student}{\textsf{Student}\xspace}
\newcommand{\spatial}{\textsf{3D-spatial}\xspace}
\newcommand{\censusbig}{\textsf{Census1990}\xspace}
\newcommand{\hmda}{\textsf{HMDA}\xspace}

\newtheorem{theorem}{Theorem}

\newtheorem{proposition}{Proposition}
\newtheorem{definition}{Definition}
\newtheorem{assumption}{Assumption}
\newtheorem{lemma}{Lemma}

\ifCLASSOPTIONcompsoc
  \usepackage[nocompress]{cite}
\else
  \usepackage{cite}
\fi
\ifCLASSINFOpdf
\else
\fi  
\begin{document}
\title{Efficient $k$-means with Individual Fairness via Exponential Tilting}
\author{Shengkun Zhu,
        Jinshan Zeng,
        Yuan Sun,
        Sheng Wang,
        Xiaodong Li,~\IEEEmembership{Fellow,~IEEE,}
        Zhiyong Peng
\IEEEcompsocitemizethanks{\IEEEcompsocthanksitem Shengkun Zhu,  Sheng Wang, and Zhiyong Peng are with the School of Computer Science, Wuhan University.
E-mail: \{whuzsk66, swangcs, peng\}@whu.edu.cn
\IEEEcompsocthanksitem Jinshan Zeng is with the School of Computer and Information Engineering, Jiangxi Normal University.
E-mail: jinshanzeng@jxnu.edu.cn
\IEEEcompsocthanksitem Yuan Sun is with La Trobe Business School, La Trobe University.
Email: yuan.sun@latrobe.edu.au.
\IEEEcompsocthanksitem Xiaodong Li is with the School of Computing Technologies, RMIT University.
Email: xiaodong.li@rmit.edu.au.
}

}

\markboth{Journal of \LaTeX\ Class Files,~Vol.~14, No.~8, June~2024}%
{Shell \MakeLowercase{\textit{et al.}}: Bare Demo of IEEEtran.cls for Computer Society Journals}
\IEEEtitleabstractindextext{
\begin{abstract}
In location-based resource allocation scenarios, the distances between each individual and the facility are desired to be approximately equal, thereby ensuring fairness. 
Individually fair clustering is often employed to achieve the principle of \textit{``treating all points equally''}, which can be applied in these scenarios.
This paper proposes a novel algorithm, tilted $k$-means (\tkm), aiming to achieve individual fairness in clustering to ensure fair allocation of resources.
We integrate the exponential tilting into the sum of squared errors (SSE) to formulate a novel objective function called tilted SSE.
We demonstrate that the tilted SSE can generalize to SSE and employ the coordinate descent and first-order gradient method for optimization.
We propose a novel fairness metric, the variance of the squared distance of each point to its nearest centroid within a cluster, which can alleviate the \textit{Matthew Effect} typically caused by existing fairness metrics.
Our theoretical analysis demonstrates that the well-known $k$-means++ incurs a multiplicative error of $O(k\log k)$ with our objective function, and we establish the convergence of \tkm under mild conditions. 
In terms of fairness, we prove that the variance in each cluster generated by \tkm decreases with $t$, where $t$ is a hyperparameter that adjusts the trade-off between utility and fairness. 
In terms of efficiency, we demonstrate the time complexity of \tkm is linear with the dataset size. Moreover, we demonstrate the monotonicity of the tilted SSE with respect to $t$ in a simple case.
Our experimental results demonstrate that \tkm outperforms state-of-the-art methods in effectiveness, fairness, and efficiency. 
Specifically, \tkm exhibits a better trade-off between clustering utility and fairness than six baselines and achieves \textit{hundreds or even thousands} of times acceleration in running time. 
Moreover, \tkm can overcome the RAM overflow issue that other methods encounter with a large dataset size.
\end{abstract}
}

\maketitle

\IEEEdisplaynontitleabstractindextext
\IEEEpeerreviewmaketitle

\section{Introduction}\label{intro}
In the era of big data, the scale of data is increasing exponentially \cite{Nargesian2019Lake, Zhang2015Memory}, emerging from diverse fields \cite{Edara2021Metadata, Fan2022Graphs} with rich information and potential value \cite{Wu2014Mining}. 
Clustering algorithms have become powerful tools for exploring the internal structure of datasets by partitioning data points into different clusters \cite{Ankerst1999OPTICS}, where data points within the same cluster are similar to each other, while those in different clusters are dissimilar \cite{Jain1999Clustering,zhou2009graph}. 
$k$-means is one of the most classic clustering algorithms, which measures the similarity between data points using Euclidean distance and is suitable for various types of data \cite{mortensen2023marigold,cao2010mining,cong2009efficient, Aken2017Automatic}. 
This characteristic makes $k$-means widely applicable in various location-based resource allocation scenarios, such as opening new facilities to serve residents \cite{Edo2020Elastic, gupta2017local, Wang2020Efficiency, Shang2016Collective}.

\begin{figure}
  \centering
  \includegraphics[width=\linewidth]{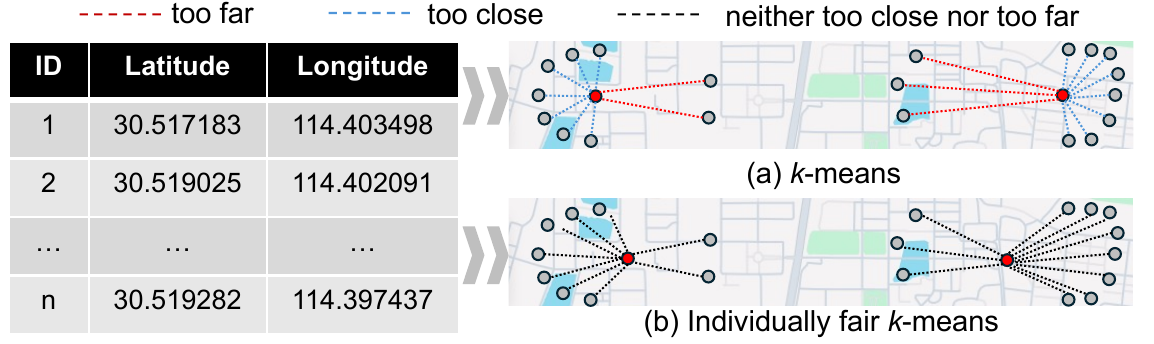}
  \caption{A comparison between $k$-means and individually fair $k$-means. $k$-means results in those minority residents being too far from the centroid, while in the clustering results of individually fair $k$-means, the distance of each resident to the centroid is approximately equal.}
  \label{fig: example}
\end{figure}

However, applying $k$-means to resource allocation scenarios may lead to the issue of unfairness \cite{Shaham2022Models,Jung2020Service,Tutorial2022Clustering}. Consider the scenario in Figure \ref{fig: example}(a): when setting up public facilities such as hospitals for residents, $k$-means tends to place these facilities closer to densely populated areas, resulting in sparsely populated areas having difficulty accessing public resources and unfair treatment for minority residents.
Individual fairness is a promising concept that can ensure that within the same cluster, each data point is treated approximately equally \cite{Mehrabi2022Survey, Simon2023Fairness}. Figure \ref{fig: example}(b) shows the clustering results obtained by $k$-means with individual fairness, where the distances from each resident to the centroid are approximately equal. In this case, we consider the clustering result to be fair.

One of the most widely studied concepts in individual fairness for $k$-clustering is the ``\textit{service in your neighborhood}" proposed by Jung et al. \cite{Jung2020Service}. 
This concept ensures that each data point has a centroid within a small constant factor of their neighborhood radius. 
The neighborhood radius is defined as the minimum radius of a ball centered at each data point that includes at least $n/k$ data points, where $n$ is the total number of data points.
Several studies \cite{Negahbani2021Better, Mahabadi2020Individual} have made significant improvements in clustering utility and yielded tighter theoretical bounds based on this individual fairness concept. Mahabadi and Vakilian \cite{Mahabadi2020Individual} introduced a local search method for $k$-clustering that notably surpasses \cite{Jung2020Service} in effectiveness. 
Negahbani and Chakrabarty \cite{Negahbani2021Better} proposed leveraging linear programming to develop improved algorithms for individually fair $k$-clustering, both theoretically and practically.

However, the fairness definition of these methods faces a similar issue. To illustrate this, let us consider the scenario in Fig. \ref{fig: radius}: point B resides in a sparsely populated area, with its neighborhood radius larger than point A situated in a densely populated area. This tends to result in opening more facilities in densely populated areas while only opening a few facilities in sparsely populated areas \cite{Negahbani2021Better}. 
Within the same radius, individual A in densely populated areas has more opportunities to choose facilities, while individual B in sparsely populated areas may have only a single facility available.
Moreover, densely populated areas attract more individuals due to abundant resources. To meet the needs of these individuals, additional facilities must be opened. This results in the development of sparse areas continually lagging behind.
This phenomenon, also known as the \textit{Matthew Effect} \cite{Robert1968Matthew, Azoulay2014Matthew}, is a sociological concept describing how the distribution of resources, wealth, and opportunities tends to favor individuals who already possess them. 

Moreover, existing individually fair clustering methods suffer from the efficiency issue: their running time heavily depends on the dataset size. The most promising theoretical finding suggests a time complexity of $O(kn^4)$ \cite{Negahbani2021Better}.
Based on data from the U.S. Census, the population of New York City was 8.468 million in 2021 \cite{uscensus}. Due to the high time complexity of existing algorithms, no individual fair clustering algorithm can effectively perform clustering analysis on such a large-scale dataset.
Moreover, as the dataset size increases, existing methods suffer from the issue of RAM overflow since they require computation of the distance between each pair data point, necessitating the storage of an $n\times n$ array in memory.
Additionally, in the clustering results obtained by these algorithms, each centroid must be selected from the data points, which is often unreasonable in real-world applications. 

\begin{figure}
  \centering
  \includegraphics[width=\linewidth]{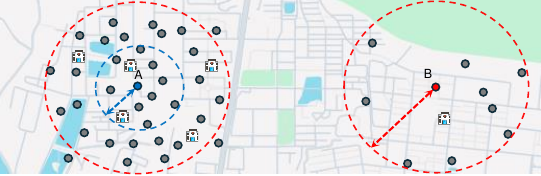}
  \caption{In sparse areas, the neighborhood radius for point B is larger than the neighborhood radius for point A in dense areas. Within the same radius, individual A can access more facilities.}
  \label{fig: radius}
\end{figure}

\textit{Exponential tilting} is a widely used technique to induce parametric shifts in distributions in various disciplines, including statistics \cite{butler2007saddlepoint,siegmund1976importance, Rong2022Bayesian}, probability \cite{dembo2009large}, information theory \cite{Merhav2014List, Beirami2019Characterization}, and optimization \cite{Pee2011Solving, Shen2010Dual}. 
Li et al. \cite{li2023jmlr} first proposed using exponential tilting in machine learning to ensure the fairness of empirical risk minimization (TERM). 
The flexibility of TERM lies in its ability to adjust the impact of individual losses using a scale parameter and thus enables us to effectively tune the influence of minority data points as required \cite{Wang2023Sparse}.
TERM provided examples of exponential tilting in supervised learning, such as linear regression and logistic regression. However, the practical applications of exponential tilting in unsupervised learning, especially in clustering algorithms, remain unresolved.
Furthermore, some theoretical analysis of TERM relies on the assumption that the objective function follows a generalized linear model, which does not hold for clustering algorithms.


We aim to utilize the ability of exponential tilting to induce parametric shifts in distribution to ensure individual fairness for clustering analysis. 
Building on this concept, we propose a novel loss function, tilted SSE, for the individually fair $k$-means problem based on exponential tilting, and suggest solving this problem effectively through coordinate descent (CD) and stochastic gradient descent (SGD), ensuring that the centroid in each cluster is closer to minority data points, thus guaranteeing individual fairness. Moreover, we demonstrate that tilted SSE can generalize to SSE when the scaled parameter in \tkm is set to 0. 
Due to the fact that existing fairness metrics may exacerbate the Matthew Effect in location-based resource allocation scenarios, we propose a novel criterion for evaluating fairness within clusters, utilizing the \textit{variance} of distances between each data point and its centroid. Our fairness metric aims to treat each individual more equitably compared to existing metrics, thereby mitigating the Matthew Effect.

Our theoretical analysis comprises five parts: approximation guarantee, convergence analysis, fairness analysis, efficiency analysis, and monotonicity analysis.
Our approximation guarantee indicates that the centroids obtained through the well-known $k$-means++ incur a multiplicative error of $O(k\log k)$.
We establish the convergence analysis for \tkm under mild conditions. 
Specifically, we demonstrate that the expected tilted SSE is non-increasing with respect to iterations.
For fairness analysis, we demonstrate that the variance of distances in each cluster decreases as the increase of hyperparameter $t$ in \tkm. 
A smaller variance indicates greater fairness, implying that as $t$ grows, clustering becomes fairer.
For efficiency analysis, we demonstrate that the time complexity of \tkm is $O(kn)$, similar to that of $k$-means.
Note that the time complexity of the state-of-the-art method is $O(kn^4)$ \cite{Negahbani2021Better}, while that of \tkm is linear with respect to $n$. Therefore, \tkm exhibits a significant advantage in efficiency compared to other methods.
For monotonicity analysis, we demonstrate that the tilted SSE monotonically increases with $t$ in a simple case. This property may guide the choices of $t$ for \tkm in practical applications.

Our experimental evaluations demonstrate the effectiveness, fairness, efficiency, and convergence of \tkm over ten real-world datasets with five measurements. Our experimental findings indicate that \tkm outperforms the state-of-the-art methods regarding the trade-off between clustering utility and fairness. 
Specifically, we use SSE to measure clustering utility. The SSE of \tkm is lower than that of the state-of-the-art method and is very close to clustering algorithms that do not consider individual fairness on some datasets. To evaluate fairness, we use not only variance as a metric but also the maximum distance from each sample point to the centroid within each cluster. 
Our results show that \tkm outperforms the state-of-the-art fair clustering methods on both metrics.
Moreover, \tkm's performance in efficiency is remarkably impressive. Due to the linear time complexity with dataset size, \tkm achieves acceleration of \textbf{hundreds or even thousands of times} compared to other fairness-aware clustering methods. Furthermore, \tkm can overcome the RAM overflow issues in other methods when dealing with large-scale data.
Additionally, we validate the impact of different hyperparameters in \tkm on its convergence.

\noindent Our contributions are summarized as follows:
\begin{itemize}
    \item We incorporate exponential tilting into SSE to propose a novel method for individually fair $k$-means: \tkm.
    \item We theoretically analyze \tkm's approximation guarantee, convergence, fairness, efficiency, and monotonicity.
    \item We experimentally validated the effectiveness, fairness, efficiency, and convergence of \tkm.

\end{itemize}

The remaining sections are structured as follows: Section \ref{sec notations} presents the notations used in this paper, Section \ref{sec related work} presents the related work, Section \ref{sce preliminary} introduces the preliminaries used in our study, Section \ref{sec proposed tilted kmeans} outlines our proposed method, \tkm, Section \ref{sec experiments} validates our algorithm through experiments, Section \ref{sec conclusion} concludes our paper, and Section \ref{sec proof} presents the proofs of our theories.

\begin{table}
\centering
\setlength\tabcolsep{15pt}
\caption{Summary of notations}
\label{tab: notation}
\begin{tabular}{ll}
\toprule
\textbf{Notation} & \textbf{Description}\\
\midrule
$\mathcal{X}:=\{\boldsymbol{x}_i\}_{i=1}^n$ & The dataset of $n$ points\\
$\mathcal{S}:=\{\mathcal{S}_j\}_{j=1}^k$ & The set of $k$ clusters\\
$\mathcal{C}:=\{\boldsymbol{c}_j\}_{j=1}^k $& The set of $k$ centroids\\
$\overline{\psi}, \overline{\phi}$ & The SSE, tilted SSE of all clusters\\
$\psi, \phi$& The SSE, tilted SSE of each cluster\\
$\m, \tm $& The arithmetic, tilted mean operator\\
$\eta$& The learning rate\\
$E$& The epoch size\\
\bottomrule
\end{tabular}
\end{table}

\section{Notations}\label{sec notations}

We use different text formatting styles to represent different mathematical concepts: plain letters for scalars, bold letters for vectors, and calligraphic letters for sets. For instance, $k$ represents a scalar, $\boldsymbol{x}$ represents a vector, and $\mathcal{C}$ denotes a set. Without loss of generality, all data points in this paper are represented using vectors.
We use $[k]$ to represent the set $\{1, 2, ..., k\}$. The symbol $\mathbb{E}$ denotes the expectation of a random variable, and we use ``$:=$" to indicate a definition. 
We use $\mathbb{I}$ to denote the identity matrix.
We use $\|\cdot\|$ to denote the Euclidean norm of a vector. We use the symbol ``$\log$'' to denote the natural logarithm with base $e$. Table \ref{tab: notation} lists the notations appearing in this paper and their interpretations.

\section{Related Work}\label{sec related work}
We provide an overview of previous studies on fair clustering and the application of exponential tilting in various fields and highlight the limitations of these studies.

\myparagraph{Fair Clustering}
Fairness in clustering algorithms is typically divided into two categories: \textit{group fairness} and \textit{individual fairness} \cite{Francois2022Towards,Simon2023Fairness,Mehrabi2022Survey,Dong2023Fairness}. The goal of group fairness is to achieve clustering of a given set of points with minimal cost while ensuring that all clusters are balanced with respect to certain protected attributes, such as gender or race. Group fairness is not the focus of this paper, so interested readers can refer to \cite{Chen2019Proportionally,Bera2019fairAlgorithms, Chierichetti2017Fair, Zhu2023F3KM,Dickerson2024Doubly}. 

The concept of \textit{individual fairness} was initially introduced by Dwork et al. \cite{Dwork2012Fairness} in the context of classification, which posits that ``\textit{similar individuals should be treated equally}''. Several studies have explored this definition of individual fairness in clustering, such as \cite{Brubach2020Pairwise,Chakrabarti2022Individually}. Another widely used and researched concept of individual fairness is referred to as ``\textit{service in your neighborhood}'', which was initially suggested by Jung et al. \cite{Jung2020Service}.
This concept aims to ensure that each data point has a centroid within at most a small constant factor of their neighborhood radius, where the neighborhood radius is the minimum radius of a ball centered at the data point $\bx_i$ that includes at least $n/k$ data points. 
Subsequently, various methods addressing the individually fair $k$-clustering were based on this paradigm \cite{Negahbani2021Better,Mahabadi2020Individual,Chhaya2022Coresets}, along with numerous improved theoretical upper bounds \cite{bateni2024scalable,Vakilian2022Improved}. 
Mahabadi and Vakilian \cite{Mahabadi2020Individual} introduced a local search algorithm for $k$-clustering, which significantly outperforms the method proposed by Jung et al. \cite{Jung2020Service} in terms of clustering utility. Negahbani and Chakrabarty \cite{Negahbani2021Better} proposed leveraging linear programming techniques to develop improved algorithms for individually fair $k$-clustering, both theoretically and practically. 
The fairness metric used in these methods can alleviate some of the unfairness in location-based resource allocation scenarios by ensuring that facilities are within a neighborhood radius of each data point. 
However, this metric might exacerbate the Matthew Effect, as it tends to result in more facilities being opened in densely populated areas while fewer facilities are opened in sparsely populated areas.
Moreover, existing individually fair clustering methods encounter the same issue: they suffer from prohibitively high computational time. Specifically, the time complexity of \cite{Mahabadi2020Individual} is $O(k^5n^4)$, and \cite{Negahbani2021Better} is $O(kn^4)$. 
To address the running time issue, Chhaya et al. \cite{Chhaya2022Coresets} proposed a method to reduce the dataset size by constructing a coreset. However, this approach results in diminished clustering utility and fails to mitigate the inherent dependency of the computational complexity of existing individual fairness clustering on dataset size.

\myparagraph{Exponential Tilting}
We elucidate the concept of exponential tilting and explore its applications across various disciplines.
Let $\mathcal{P}:=\{p_{\theta}\}$ be a set of probability distributions with parameter $\theta$, $X$ denote a random variable drawn from the probability distribution $p_{\theta}$, then for any $x\in\mathcal{X}$, the information of $x$ under $\theta$ \cite{cover1999elements} is defined as
\begin{align}\label{eq 1}
    f(x,\theta):=\log p_{\theta}(X=x).
\end{align}

When $\theta$ is not specified, we assume that $X$ is a random variable drawn from the distribution $p(\cdot)$. Then the cumulant generating function of $f(X,\theta)$ \cite{dembo2009large} is defined as
\begin{align}\label{eq: 2}
    \Delta_{X}(t,\theta):=\log\Bigl(\mathbb{E}\Bigl[e^{tf(X,\theta)}\Bigr]\Bigr)=\log\sum_{x}p(x)p_{\theta}(x)^{-t},
\end{align}    
where $\mathbb{E}[e^{tf(X,\theta)}]$ is commonly referred to as an exponential tilting of the information density, and can induce the probability distribution with parameter $\theta$ shifting. 
Exponential tilting has been applied in numerous fields, such as statistics \cite{butler2007saddlepoint,siegmund1976importance,Rong2022Bayesian}, applied probability \cite{dembo2009large}, information theory \cite{Merhav2014List,Beirami2019Characterization}, and optimization \cite{Pee2011Solving,Shen2010Dual}. Interested readers can refer to \cite{li2023jmlr} for a more detailed introduction.
Currently, there are relatively few applications of exponential tilting in machine learning \cite{li2023jmlr,Wang2023Sparse,Szab2021Cross}.
Li et al. \cite{li2023jmlr} proposed tilted empirical risk minimization (TERM), which allows flexible tuning of individual losses, marking a pioneering move in machine learning.
TERM offers several examples of supervised learning, including linear regression and logistic regression, as illustrated in Fig. \ref{fig: toy example in tian}. Recent research has also concentrated on supervised learning, such as the additive model \cite{Wang2023Sparse} and semantic segmentation \cite{Szab2021Cross}.

\noindent\underline{\textit{\textbf{Remarks}}}.
1) Current fairness metrics might exacerbate the Matthew Effect, as they tend to lead to more facilities being opened in densely populated areas while fewer facilities are opened in sparsely populated areas.
2) The efficiency of existing individually fair $k$-clustering algorithms heavily depends on the number of samples $n$ of the dataset.
3) Existing individually fair clustering algorithms cannot flexibly tune the trade-off between utility and fairness. Moreover, these algorithms require cluster centroids to be one of the data points.
4) The current application of exponential tilting is still limited to supervised learning, and it has not been applied in unsupervised learning, especially in clustering.

\begin{figure}
  \centering
  \includegraphics[width=\linewidth]{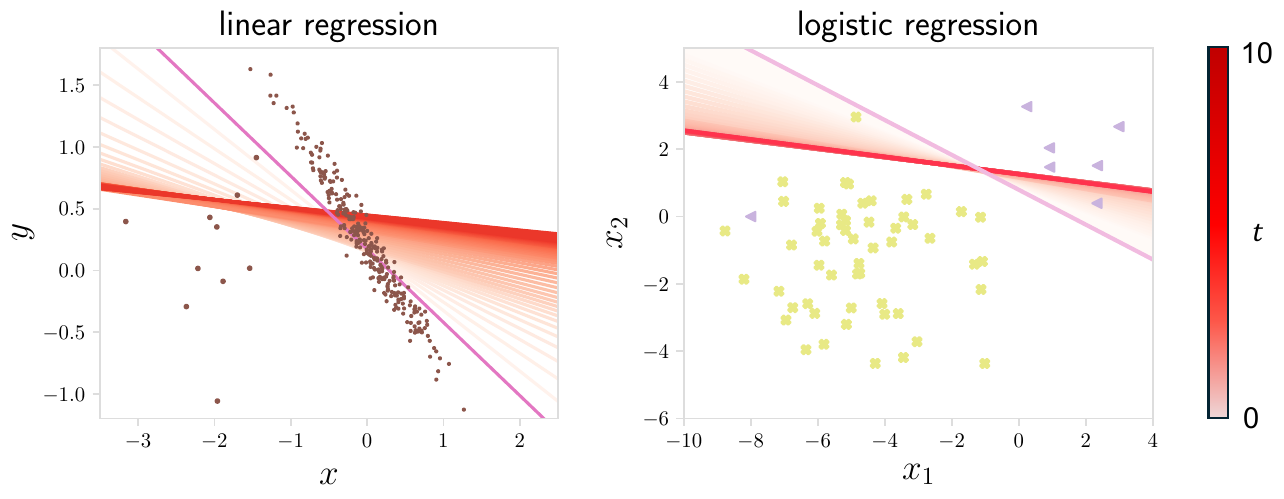}
  \caption{Two examples of TERM \cite{li2023jmlr}. Increasing parameter $t$ can magnify the impact of minority points on the models.}
  \label{fig: toy example in tian}
\end{figure}

\section{Preliminaries}\label{sce preliminary}
We begin by introducing the definition of $k$-means. Then, we present the well-known $k$-means++ initialization method.

\subsection{$k$-means}
$k$-means is a widely used clustering algorithm designed to partition a dataset into $k$ distinct clusters based on similarities among data points. Let $\mathcal{X}:=\{\boldsymbol{x}_i\}_{i=1}^n$ be a dataset of $n$ points, $k$-means aims to find a set $\mathcal{S}:=\{\mathcal{S}_j\}_{j=1}^k$ of $k$ clusters such that the sum of squared error (SSE) is minimized, 
\begin{align}\label{eq: kmeans sse}
    \min_{\mathcal{S},\mathcal{C}}\Bigl\{\overline{\psi}(\mathcal{S},\mathcal{C}):=\sum_{j=1}^{k}\frac{1}{n}\sum_{\bx_i\in\mathcal{S}_j}f(\bx_i,\bc_j)\Bigr\},
\end{align}
where $\mathcal{C}:=\{\bc_j\}_{j=1}^k$ is a set of centroids, $\boldsymbol{c}_j$ is the centroid of cluster $\mathcal{S}_j$, $f(\bx_i,\bc_j):=\|\bx_i-\bc_j\|^2$ denotes the square of the Euclidean distance from a data point $\bx_i$ to the centroid $\bc_j$.
The commonly used method for solving $k$-means is the well-known Lloyd's heuristic \cite{Lloyd1982}, which iteratively computes the assignment of each data point and the centroids through coordinate descent. Next, we provide a detailed description of the optimization process of Lloyd's heuristic. We begin by presenting the equivalent form of Problem (\ref{eq: kmeans sse}) as
\begin{align}\label{eq: equivalent kmeans sse}
    \min_{\mathcal{S},\mathcal{C}}\Bigl\{\overline{\psi}(\mathcal{S},\mathcal{C})\!:=\!\sum_{j=1}^{k}\psi(\bdelta_j,\bc_j)\!:=\!\sum_{j=1}^{k}\frac{1}{n}\sum_{i=1}^{n}f(\bx_i,\bc_j) \delta_{ij}\Bigr\},\!\!
\end{align}
where $\delta_{ij},i\in[n],j\in[k]$ denotes the assignment of each data point, for example, if $\boldsymbol{x}_i\in \mathcal{S}_j$, then $\delta_{ij}=1$, else $\delta_{ij}=0$, 
$\boldsymbol{\delta}_j := (\delta_{1j}, \delta_{2j}, \cdots, \delta_{nj}) \in \mathbb{R}^n$ represents the assignment of data points in the $j$-th cluster, and $\psi(\bdelta_j,\bc_j):=\frac{1}{n}\sum_{i=1}^n f(\bx_i,\bc_j)\delta_{ij}$ denotes the SSE in the cluster $\mathcal{S}_j$.
To solve Problem (\ref{eq: equivalent kmeans sse}), one may iteratively assign each point to its nearest centroid and refine $\bc_j$ using Lloyd's heuristic. Following initialization, with $\bc_j$ holds constant, the solution for $\delta_{ij}$ can be obtained as
\begin{align}\label{eq: solution to fij}
    \delta_{i j}= \begin{cases}1, &j=\arg \min _l f(\bx_i,\bc_l), \\ 0, & \text{otherwise.}\end{cases}
\end{align}
When $\delta_{ij}$ holds constant, solve for $\bc_j$ as follows:
\begin{align}\label{eq: solution to cj}
    \boldsymbol{c}_j=\m(\mathcal{S}_j)=\frac{\sum_{i=1}^{n}\delta_{ij}\cdot\bx_i }{\sum_{i=1}^{n} \delta_{ij}},
\end{align}
where $\m(\cdot)$ is an operator to calculate the weighted mean.

\begingroup
\begin{algorithm}[t]
  \caption{$k$-means++}
  \label{alg: kmeans++}
  \SetAlgoLined
  \KwIn{$\mathcal{X}:=\{\bx_i\}_{i=1}^n$, $k$.}
  $\mathcal{C}\leftarrow$ Sample a point uniformly from $\mathcal{X}$\label{alg: kmeans++, step 1}\;

  $\mathcal{C}\leftarrow$ Sample the next centroid $\bc_j \in\mathcal{X}$ with probability $\frac{D(\bc_j)^2}{\sum_{\bc_j\in  \mathcal{X}} \! D(\bc_j)^2}$\label{alg: kmeans++, step 2}\;

  Repeat Step \ref{alg: kmeans++, step 2} until $k$ centroids are chosen\label{alg: kmeans++, step 3}\;
  \tcp{Coordinate descent.}
  \While{not converge\label{alg: kmeans++, step 4}}{
    Update $\delta_{ij}, i\in[n],j\in[k]$ by Equation (\ref{eq: solution to fij})\;

    Update $\bc_j,j\in[k]$ by Equation (\ref{eq: solution to cj})\label{alg: kmeans++, step 6}\;

  }
      \KwRet $\mathcal{C}$.

\end{algorithm} 




  

\endgroup

\subsection{$k$-means++}

$k$-means++ is an improved version of $k$-means by providing a more effective strategy for selecting initial centroids, thus enhancing the speed and accuracy \cite{arthur2007k}. We provide the details of $k$-means++ in Algorithm \ref{alg: kmeans++}.
Its process involves selecting the first centroid randomly from the dataset (Step \ref{alg: kmeans++, step 1} in Algorithm \ref{alg: kmeans++}). Let $D(\bx_i)$ be the shortest distance from a data point $\bx_i$ to its closest centroids that we have already chosen. The subsequent centroid is chosen from the data points based on their squared distances to the nearest existing centroids, with a probability $\frac{D(\bx')^2}{\sum_{\bx_i \in \mathcal{X}}D(\bx_i)^2}$ (Step \ref{alg: kmeans++, step 2} in Algorithm \ref{alg: kmeans++}). This iterative process is repeated until $k$ centroids are chosen (Step \ref{alg: kmeans++, step 3} in Algorithm \ref{alg: kmeans++}).
After selecting $k$ centroids, the subsequent update of $\bdelta_j$ and $\bc_j$ is performed through coordinate descent, which is identical to Lloyd's heuristic (Steps \ref{alg: kmeans++, step 4}-\ref{alg: kmeans++, step 6} in Algorithm \ref{alg: kmeans++}).

\section{Proposed \tkm}\label{sec proposed tilted kmeans}
In this section, we begin by proposing the objective function of tilted $k$-means (\tkm) and presenting the corresponding optimization method. Then, we theoretically analyze the convergence, approximation guarantee, fairness, efficiency, and monotonicity of \tkm.


\subsection{Objective Function of \tkm}
Due to the characteristic of exponential tilting inducing parametric shifts in distributions, we consider incorporating exponential tilting into SSE to obtain \textit{tilted SSE}.
The objective of tilted $k$-means is to minimize the tilted SSE within each cluster as follows:
\begin{align}\label{eq: tkm}
    \min_{\mathcal{S},\mathcal{C}}\Bigl\{\overline{\phi}(t,\mathcal{S},\mathcal{C}):=&\sum_{j=1}^{k}\phi(t,\bdelta_j,\bc_j)\notag\\
    :=&\sum_{j=1}^{k}\frac{1}{t}\log\frac{1}{n}\!\sum_{i=1}^{n}e^{tf(\bx_i,\bc_j)\delta_{ij}}\Bigr\},
\end{align}    
where $t>0$ is a hyperparameter. 
Note that when $\mathcal{P}$ in (\ref{eq 1}) is an exponential set of distributions parameterized by $\bc_j$ and $\bdelta_{j}$, the cumulant generating function can be written as:
\begin{align}\label{eq: cumulant function}
    \Delta_{X}(t,\bdelta_{j},\bc_j):=\log\frac{1}{n}\sum_{i=1}^{n}e^{tf(\bx_i,\bc_j)\delta_{ij}}.
\end{align}
Therefore, it is clear that the objective function of \tkm can be considered as a properly scaled summation version of the cumulant generating function in Equation~\eqref{eq: cumulant function}.

Next, we consider the case of $t = 0$ in \tkm.
When $t \to 0$, according to L'Hôpital's rule, it holds that:
\begin{align}\label{eq: limition}
    \lim_{t\to 0}\overline{\phi}(t,\mathcal{S},\mathcal{C}) = \frac{1}{n}\sum_{j=1}^k\sum_{i=1}^{n}f(\bx_i,\bc_j)\delta_{ij}.
\end{align}
Therefore, when $t \to 0$, tilted SSE generates to SSE. Without loss of generality, we define
\begin{align}\label{eq: phi 0 > psi}
    \phi(0,\bdelta_{j},c_j):= \psi(\bdelta_{j},\bc_j).
\end{align}

\begin{algorithm}[t]
  \caption{Solving tilted $k$-means via SGD}
  \label{alg: sgd for tilted kmeans}
  \SetAlgoLined
  \KwIn{$\mathcal{X}:=\{\bx_i\}_{i=1}^n$, $k$: \# of clusters, $E$: \# of epoch.}
  Initialize $\mathcal{C}:=\{\bc_j\}_{j=1}^k$ by $k$-means++\;\label{line 1 alg 3}
  \While{not converge}{
  \tcc{Assignment.}
    Update $\delta_{ij}, i\in[n],j\in[k]$ by (\ref{eq: solution to fij})\; \label{line 3 alg 3}
    \tcc{Refinement.}
    \For{$e=1,\cdots,E$}{
    
    Sample a mini-batch data $\mathcal{B}$ from $\mathcal{X}$\; \label{line 5 alg 3}
    
    Update $\bc_j,j\in[k]$ by (\ref{eq: sgd for cj})\; \label{line 7 alg 3}

    }

  }
      \KwRet $\mathcal{C}$.

\end{algorithm}

\subsection{Solving Tilted $k$-means}
Since Problem (\ref{eq: tkm}) involves a highly non-convex objective function with multi-block variables, we consider using coordinate descent (CD) to solve it. We begin by fixing $\bc_j$ to solve $\delta_{ij}$. 
Due to the monotonically increasing nature of the objective function with respect to $tf(\bx_i,\bc_j)$, the solution for $\delta_{ij}$ is identical to that of Equation (\ref{eq: solution to fij}).
Next, we consider fixing $\delta_{ij}$ to solve $\mathcal{C}$.
Since the tilted SSE is convex with respect to $\bc_j$ (this property will be proven in Section \ref{sec: Theoretical Analysis}), we can derive the optimality condition for the tilted SSE with respect to $\bc_j$. We then present the first-order gradient of $\overline{\phi}(t,\mathcal{S},\mathcal{C})$ with respect to $\bc_j$ as follows,
\begin{equation}\label{eq8}
\begin{split}
    \nabla_{\bc_j} \overline{\phi}(t,\mathcal{S},\mathcal{C})=&\nabla_{\bc_j} \phi(t,\bdelta_j,\bc_j) \\
    = &\frac{\sum_{\bx_i\in\mathcal{S}_j} \!e^{tf(\bx_i,\bc_j)}\cdot\nabla_{\bc_j}f(\bx_i,\bc_j)}{\sum_{i=1}^n e^{tf(\bx_i,\bc_j)\delta_{ij}}}.
\end{split}
\end{equation}    
where $\nabla_{\bc_j}f(\bx_i,\bc_j):=2(\bx_i-\bc_j)$ is the first-order gradient of $f(\bx_i,\bc_j)$ with respect to $\bc_j$.
Then setting Equation (\ref{eq8}) equal to zero yields the optimal condition of $\bc_j$: 
\begin{align}\label{eq9}
    \sum_{\bx_i\in\mathcal{S}_j} e^{t\|\bx_i-\bc_j\|^2}(\bx_i-\bc_j)=0.
\end{align}

We define a tilted mean operator $\tm(\cdot)$, where $\bc_j=\tm(t,\mathcal{S}_j)$ represents the values of $\bc_j$ that satisfy Equation (\ref{eq9}). Note that obtaining the closed solution for $c_j$ from Equation (\ref{eq9}) is nontrivial, therefore, we employ the first-order gradient method to solve $\bc_j$. Let $\mathcal{B}$ be a batch data from $\mathcal{X}$, then $\bc_j$ is updated as follows:
\begin{align}\label{eq: sgd for cj}
    \bc_j &\leftarrow \bc_j - \eta\cdot \nabla_{\bc_j} \overline{\phi}(t,\mathcal{B},\mathcal{C}),\\
    \nabla_{\bc_j} \overline{\phi}(t,\mathcal{B},\mathcal{C})&=\frac{\sum_{\bx_i\in\mathcal{B}_j}e^{tf(\bx_i,\bc_j)}\cdot\nabla_{\bc_j}f(\bx_i,\bc_j)}{\sum_{i=1}^{n}e^{tf(\bx_i,\bc_j)\delta_{ij}}},
\end{align}    
where $\eta$ is a learning rate, and $\mathcal{B}_j:=\mathcal{S}_j\cap \mathcal{B}$. 
Note that the first-order gradient method is a commonly used optimization method for solving such problems. Interested readers may consider trying second-order gradient methods such as \textit{Newton method} \cite{Nesterov2004Introductory} for solving $\bc_j$.

\myparagraph{Algorithm Description}
The algorithmic process of \tkm can be summarized into three parts: initialization, assignment, and refinement.
We provide algorithm details for \tkm in Algorithm \ref{alg: sgd for tilted kmeans} and an example in Fig. \ref{fig: example of tkm}. Firstly, the centroids set $\mathcal{C}$ is initialized using $k$-means++ (Line \ref{line 1 alg 3} in Algorithm \ref{alg: sgd for tilted kmeans}). Subsequently, we employ CD to iteratively solve $\delta_{ij}$ (assignment) and $\bc_j$ (refinement) (Lines \ref{line 3 alg 3}-\ref{line 7 alg 3} in Algorithm \ref{alg: sgd for tilted kmeans}). We set $E$ epochs for solving $\bc_j$, where in each epoch, a batch $\mathcal{B}$ data is sampled from $\mathcal{X}$, and the data points within $\mathcal{B}\cap \mathcal{S}_j$ are used to solve $\bc_j$ using Equation~\eqref{eq: sgd for cj}.

\begin{figure}
  \centering
  \includegraphics[width=\linewidth]{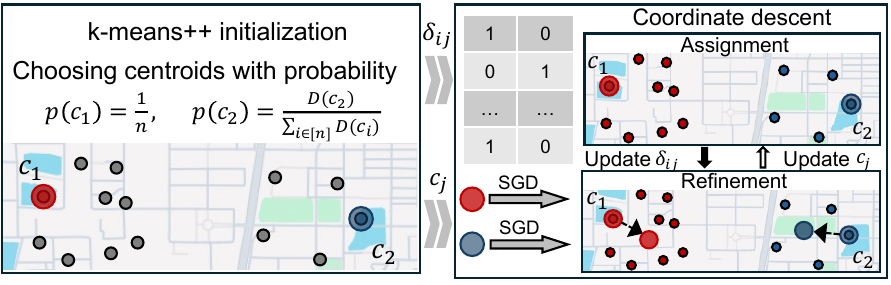}
  \caption{An example of \tkm includes the stages of initialization, assignment, and refinement.}
  \label{fig: example of tkm}
\end{figure}

\subsection{Theoretical Analysis}\label{sec: Theoretical Analysis}
Our theoretical analysis consists of five parts. The first part provides an approximation guarantee for the initial centroids obtained by $k$-means++ with respect to the tilted SSE. Then we present a convergence analysis of \tkm. Next, we delve into a fairness analysis of \tkm. In the fourth part, we explore the time complexity of \tkm. Finally, we analyze the monotonicity of the tilted SSE using a simple case.

\subsubsection{Definitions and Assumptions}
We begin by providing some definitions and assumptions used throughout our theories.

\begin{definition}[Tilted weight]
    Given a cluster $\mathcal{S}_j$ and a centroid $\bc_j$, the tilted weight $w_i(t,\mathcal{S}_j,\bc_j)$ of a data point $\bx_i$ is defined as 
    \begin{equation}
    \begin{split}
        w_i(t,\mathcal{S}_j,\bc_j):=&\,\,\frac{e^{tf(\bx_i,\bc_j)}}{\sum_{\bx_i\in\mathcal{S}_j} e^{tf(\bx_i,\bc_j)}}.
    \end{split}
    \end{equation}

\end{definition}

\begin{definition}[Tilted empirical mean and variance] 
Let $\mathrm{f}(\mathcal{S}_j,\bc_j):=\bigl\{f(\bx_i,\bc_j)|\bx_i\in\mathcal{S}_j\bigl\}$ be a set of squared Euclidean distances of points in $\mathcal{S}_j$ to the centroid $\bc_j$, then the tilted empirical mean and variance in the cluster $\mathcal{S}_j$ are defined as 
\begin{align}
    \mathbb{E}_t\bigl(\mathrm{f}(\mathcal{S}_j,\bc_j)\bigr)&:=\sum_{\bx_i\in\mathcal{S}_j}w_i(t,\mathcal{S}_j,\bc_j)\cdot f(\bx_i,\bc_j),\\
    \var_t\bigl(\mathrm{f}(\mathcal{S}_j,\bc_j)\bigr)&:=\mathbb{E}_t\Bigl(f(\bx_i,\bc_j)-\mathbb{E}_t\bigl(\mathrm{f}(\mathcal{S}_j,\bc_j)\bigr)\Bigr)^2.
\end{align}    
\end{definition}

Note that when $t=0$, tilted empirical mean and variance generalize to the standard mean and variance in statistics.

\begin{definition}[Gradient Lipschitz Continuity]
    The objective function $f: \mathbb{R}^d\rightarrow\mathbb{R}$ is continuously differentiable and the gradient function of $f$,
namely, $\nabla f: \mathbb{R}^d\rightarrow\mathbb{R}$, is gradient Lipschitz continuous with Lipschitz constant $L>0$, if for any $\bc,\bc'\in\mathbb{R}^d$, it holds that
\begin{align}
    \|\nabla f(\bc)-\nabla f(\bc')\|\leq L \|\bc-\bc'\|.
\end{align}
\end{definition}

\begin{definition}[Tilted Hessian]
    For any $t\geq 0$, we define the Tilted Hessian $\nabla^2_{\bc_j\bc_j^{\top}}\phi(t,\bdelta_j,\bc_j)$ as the Hessian of $\phi(t,\bdelta_j,\bc_j)$ with respect to $\bc_j$. That is
\begin{small}
\begin{align*}
    &\!\nabla^2_{\bc_j\bc_j^{\top}}\phi(t,\bdelta_j,\bc_j)\!=\frac{t}{n}\sum_{i=1}^n\bigl(\nabla_{\bc_j}f(\bx_i,\bc_j)\delta_{ij}-\nabla_{\bc_j}\phi(t,\bdelta_j,\bc_j)\bigl)\notag\\
    &\!\!\!\bigl(\nabla_{\bc_j}f(\bx_i,\bc_j)\delta_{ij}\!-\!\nabla_{\bc_j}\phi(t,\bdelta_j,\bc_j)\bigl)^{\top} \!e^{t\bigl(f(\bx_i,\bc_j)\delta_{ij}-\phi(t,\bdelta_j,\bc_j)\bigr)}\\
    &+\frac{1}{n}\sum_{i=1}^n e^{t\bigl(f(\bx_i,\bc_j)\delta_{ij}-\phi(t,\bdelta_j,\bc_j)\bigr)} \cdot2\mathbb{I}\delta_{ij},
\end{align*}      
\end{small}
and $\mathbb{I}$ is an identity matrix of appropriate size.
\end{definition}

\begin{lemma}[Strong Convexity of Tilted SSE \cite{li2023jmlr}]\label{lemma: strong convexity of tilted sse}
    For any $t\geq 0$, the tilted SSE is strongly convex with respect to $\bc_j$. That is
\begin{align*}
    \nabla^2_{\bc_j\bc_j^{\top}}\phi(t,\bdelta_j,\bc_j)\succeq \frac{2|\mathcal{S}_j|}{n}\mathbb{I}.
\end{align*}
\end{lemma}
\begin{proof}
    Note that the first term in tilted Hessian is positive semi-definite, and the second term is positive definite and lower bounded by $\frac{2|\mathcal{S}_j|}{n}\mathbb{I}$, which completes the proof.
\end{proof}

\begin{lemma}[Gradient Lipschitz Continuity of Tilted SSE \cite{li2023jmlr}]\label{lemma: gradient lipschitz of tilted sse}
    For any $t\geq 0$, $\nabla_{\bc_j}\phi(t,\bdelta_j,\bc_j)$ is $L(t)$-Lipschitz with respect to $\bc_j$, where $L(t):=\sigma_{\max}\Bigl(\nabla^2_{\bc_j\bc_j^{\top}}\phi(t,\bdelta_j,\bc_j)\Bigr)$, and $\sigma_{\max}$ denotes the largest eigenvalue.
  
\end{lemma}






\begin{assumption}\label{assumption: expetation}
    Let $g(\mathcal{B},\bc_j):=\overline{\phi}(t,\mathcal{B},\mathcal{C})$ denote the mini-batch gradient of $\overline{\phi}(t,\mathcal{S},\mathcal{C})$, then the following conditions hold:
\begin{itemize}
    \item There exist scalars $\mu_G\geq \mu>0$ such that for any $\bc_j\in\mathbb{R}^d$,  
    \begin{align*}
        \nabla_{\bc_j}\phi(t,\bdelta_j,\bc_j)^{\top}\mathbb{E}[g(&\mathcal{B},\bc_j)]\geq\mu\cdot\|\nabla_{\bc_j}\phi(t,\bdelta_j,\bc_j)\|^2,\\
        \|\mathbb{E}[g(\mathcal{B},\bc_j)]\|&\leq \mu_G \cdot\|\nabla_{\bc_j}\phi(t,\bdelta_j,\bc_j)\|.
    \end{align*}
    \item There exist scalars $\nu\geq 0$ and $\nu_H\geq 0$ such that for any $\bc_j\in\mathbb{R}^d$, it holds that
    \begin{align*}
        \mathbb{E}[\|g(\mathcal{B},\bc_j)-\mathbb{E}[g(\mathcal{B},\bc_j)]\|^2]\leq \nu + \nu_H\cdot\|\nabla_{\bc_j}\phi(t,\bdelta_j,\bc_j)\|^2.
    \end{align*}
\end{itemize}
\end{assumption}


The first requirement in Assumption \ref{assumption: expetation} states that in expectation, the vector $-g(\mathcal{B},\bc_j)$ is a direction of sufficient descent for $\phi$ from $\bc_j$ with a norm comparable to the norm of the gradient. The second requirement in Assumption \ref{assumption: expetation}, states that the variance of $g(\mathcal{B},\bc_j)$ is restricted, but in a relatively minor manner.

\subsubsection{Approximation Guarantee}
Let $\overline{\phi}^*$ represent the optimal value of $\overline{\phi}$, we aim to prove that $k$-means++ can ensure the resulting initial centroids set $\mathcal{C}$ satisfy 
$\mathbb{E}[\overline{\phi}(t,\mathcal{S},\mathcal{C})]\leq\alpha\cdot\overline{\phi}^*$, where $\alpha$ is a multiplicative error. Next, we mathematically obtain the value of $\alpha$.

\begin{theorem}\label{theorem 1}
    Let $\overline{\phi}^*$ be the optimal value of tilted SSE, Let $\overline{\psi}^{\star}$ be the optimal value of SSE, then for any dataset $\mathcal{X}$, centroids set $\mathcal{C}$ initialized by $k$-means++, and induced clusters $\mathcal{S}$, it holds that 
    \begin{align}
        \mathbb{E}[\overline{\phi}(t,\mathcal{S},\mathcal{C})]\leq O(k \log k)\cdot\overline{\psi}^{\star}\leq O(k \log k)\cdot\overline{\phi}^*.
    \end{align}

\end{theorem}
The proof of Theorem \ref{theorem 1} can be found in Section \ref{sec: proof of theorem 1}.
$k$-means++ has been proven to generate initial centroids with a multiplicative error of $O(\log k)$ in $k$-means when \textit{fairness constraints are not considered} \cite{arthur2007k}. 
Theorem \ref{theorem 1} demonstrates that \textit{with individual fairness constraints}, $k$-means++ achieves the multiplicative error of $O(k \log k)$.

\subsubsection{Convergence Analysis}
Next, we provide the convergence analysis of \tkm by proving that the assignment and refinement steps ensure that the expected value of the tilted SSE decreases.
\begin{theorem}\label{theorem: convergence}
    Let $\mathcal{S}^{it}$, $\mathcal{C}^{it}$ and $\mathcal{S}^{it+1}$, $\mathcal{C}^{it+1}$ be the solutions in the $it$-th and $(it+1$)-th iterations of \tkm. Under Assumption \ref{assumption: expetation} and by choosing the learning rate $\eta<\frac{2}{\mu\cdot L(t)}$, it holds that
\begin{align}
    \mathbb{E}[\overline{\phi}(t,\mathcal{S}^{it+1},\mathcal{C}^{it+1})]\leq \overline{\phi}(t,\mathcal{S}^{it},\mathcal{C}^{it}).
\end{align}
\end{theorem}
The proof of Theorem \ref{theorem: convergence} is provided in Section \ref{proof of theorem convergence}. 
Theorem \ref{theorem: convergence} demonstrates that with the selection of an appropriate learning rate, the expected value of the tilted SSE can decrease until reaching convergence.

\subsubsection{Fairness Analysis}
We propose using the variance of each data point's squared distance to the centroid within each cluster to measure the fairness of clustering algorithms. 
Note that when $t=0$ in the tilted weight, the tilted empirical variance generalizes to standard variance. 
We employ variance as a measure of fairness because it quantifies the extent to which sample points in a dataset are distributed around the mean, with smaller variance indicating reduced fluctuation in distances from the mean and thus greater fairness.
Next, we consider the monotonicity of the tilted empirical variance with $t$.

\begin{theorem}\label{theorem: variance decrease with t}
    For any cluster $\mathcal{S}_j$, any corresponding centroid $\bc_j(t)=\tm(t,\mathcal{S}_j)$, and any $t\geq 0$, suppose all data points are normalized to a unit norm, then it holds that
\begin{align}
    \frac{\partial }{\partial t}\Bigl\{\var_\tau\Bigl(\mathrm{f}(\mathcal{S}_j,\bc_j(t))\Bigr)\Bigr\}< 0.
\end{align}
\end{theorem}

The proof of Theorem \ref{theorem: variance decrease with t} is provided in Section \ref{sec: proof of theorem fairness}. Note that $\tau$ is a constant in the calculation of tilted empirical variance, where it contributes to the tilted weight adjustment.
Theorem \ref{theorem: variance decrease with t} states that the $\tau$-tilted empirical variance among the distances between each data point in $\mathcal{S}_j$ and their corresponding centroid will decrease with an increase in $t$. 
Therefore, there exists a potential trade-off between SSE and variance, enabling solutions to flexibly achieve desirable clustering utility and fairness. 
While Theorem \ref{theorem: variance decrease with t} suppose all data points are normalized to a unit norm which is not satisfied in some datasets, we observe favorable numerical results motivating the extension of these results beyond the cases that are theoretically studied in this paper.

\subsubsection{Time Complexity}
We provide the time complexity of \tkm and analyze why \tkm is suitable for individually fair clustering analysis in big data scenarios.
\begin{theorem}\label{theorem: time complexity}
    The time complexity of \tkm is $O(nkdET)$, where $d$ is the number of attributes of each data point, $E$ is the epoch size, and $T$ is the total number of iterations.
\end{theorem}

The proof of Theorem \ref{theorem: time complexity} is provided in Section \ref{sec: proof of theorem complexity}. Note that the time complexity of \tkm is linear with the dataset size, which is the same as that of vanilla $k$-means algorithms without fair constraints such as \lloyd \cite{Lloyd1982} and SGD-based $k$-means \cite{Sculley2010Web}. In contrast, existing individual fair clustering methods exhibit a time complexity of $O(kn^4)$ \cite{Mahabadi2020Individual,Negahbani2021Better}. In the context of big data, employing these methods for clustering becomes impractical, as the required running time becomes difficult to estimate when dealing with dataset sizes reaching the order of millions. 
Moreover, these methods encounter RAM overflow issues due to the necessity of computing distances between each data point, requiring storage of an $n\times n$ array in RAM. Conversely, \tkm only necessitates distance calculations between each data point and corresponding centroids during the assignment, thus requiring the computation of only an $n\times k$ array, effectively mitigating the risk of RAM overflow.

\subsubsection{Monotonicity Analysis}
In this section, we provided a monotonicity analysis for tilted SSE in a simple case.
\begin{theorem}\label{theorem: monotonicity}
    When $k=1$, suppose all data points are normalized to a unit norm, then for any $t\geq 0$, it holds that,
\begin{align}
    \frac{\partial \overline{\phi}(t,\mathcal{S},\mathcal{C})}{\partial t}\geq0.
\end{align} 
\end{theorem}

Proof of Theorem \ref{theorem: monotonicity} is provided in Section \ref{sec: proof of theorem monotonicity}.  When $k=1$, $k$-means simplifies to a point estimation problem. In this case, Theorem \ref{theorem: monotonicity} shows that the tilted SSE increases as $t$ increases. 
While the monotonicity of the tilted SSE is restricted to the scenario when $k=1$, our experiments suggest that the tilted SSE also exhibits a monotonically increasing trend for other values of $k$.

\section{Experiments}\label{sec experiments}

\begin{table}[t]
\centering
\renewcommand{\arraystretch}{1.2}
\setlength\tabcolsep{5pt}
\caption{An overview of the datasets}
\label{datasets}
\begin{tabular}{ccccc}
\toprule
\textbf{Ref.} & \textbf{Datasets}     & \begin{tabular}[c]{@{}c@{}}\textbf{\#}\\ \textbf{of points}\end{tabular} & \begin{tabular}[c]{@{}c@{}}\textbf{\#}\\ \textbf{of attributes}\end{tabular} & \begin{tabular}[c]{@{}c@{}}\textbf{\#}\\ \textbf{of clusters}\end{tabular} \\ \midrule
\cite{HuangJV19}    & \texttt{Athlete}     & 271,117                                               & 15                                                         & 3-10                                                     \\
\cite{Moro2014bank}    & \texttt{Bank}        & 4,521                                                 & 16                                                         & 3-10                                                     \\
\cite{Kohavi1996census}    & \texttt{Census}      & 32,561                                                & 15                                                         & 3-10                                                     \\
\cite{strack2014impact}    & \texttt{Diabetes}    & 101,766                                               & 24                                                         & 3-10                                                     \\
\cite{recruitment}    & \texttt{Recruitment} & 4,001                                                 & 50                                                         & 3-10                                                     \\
\cite{spanish}    & \texttt{Spanish}     & 4,747                                                 & 15                                                         & 3-10                                                     \\
\cite{student}    & \texttt{Student}     & 32,594                                                & 21                                                         & 3-10                                                     \\
\cite{3dspatial}    & \texttt{3D-spatial}  & 434,874                                               & 12                                                         & 3-10                                                     \\
\cite{Meek2002census1990}    & \texttt{Census1990}  & 2,458,285                                             & 69                                                         & 4                                                        \\
\cite{HMDA}   & \texttt{HMDA}        & 5,986,660                                             & 53                                                         & 4                                                        \\
\cite{Fabian2011scikit}   & \texttt{Synthetic}   & 200                                                   & 2                                                          & 2, 3                                                     \\ \bottomrule
\end{tabular}
\end{table}

\myparagraph{Goals} In this section, we verify the effectiveness and efficiency of \tkm by comparing it with various methods. We also examine the impact of various hyperparameters on the convergence of \tkm. Moreover, we provide visualizations of the centroids' variations with varying $t$.
\subsection{Settings}
\myparagraph{Datasets} We employ ten real-world datasets and two synthetic datasets to validate the performance of \tkm. 
To compare the effectiveness and fairness of \tkm with various methods and parameters, we utilize \athlete, \bank, \census, \diabetes, \recruitment, \spanish, \student, and \spatial. To compare the efficiency of \tkm with other methods, we employ \censusbig and \hmda. For visualizing \tkm, we use two synthetic datasets. 
We sampled numerical features from ten real-world datasets and then standardized these features (the names of these features are provided in our repository). 
A comprehensive overview of the datasets can be obtained from Table \ref{datasets}.


\myparagraph{Baselines} We experimentally evaluate the performance of \textsf{TKM} against six methods, namely, \kplus \cite{arthur2007k}, \jkl \cite{Jung2020Service}, \mv \cite{Mahabadi2020Individual}, \fr \cite{Negahbani2021Better}, \sfr \cite{Negahbani2021Better}, and \nf \cite{Sculley2010Web}.
As explained in our related works, \jkl first introduced the concept of individual fairness for $k$-means. \mv, \fr and \sfr are three state-of-the-art methods for individually fair $k$-means. Note that \sfr is a sparsed version of \fr. 
\kplus and \nf are two clustering methods that do not take individual fairness into account. It is worth noting that \nf is a method different from the classical \lloyd. It is solved through SGD and can be considered as the case of $t=0 $ in \tkm.
For \tkm and \nf, we employed \kplus for initialization.

\begin{figure*}[p]
  \centering
  \includegraphics[width=\linewidth]{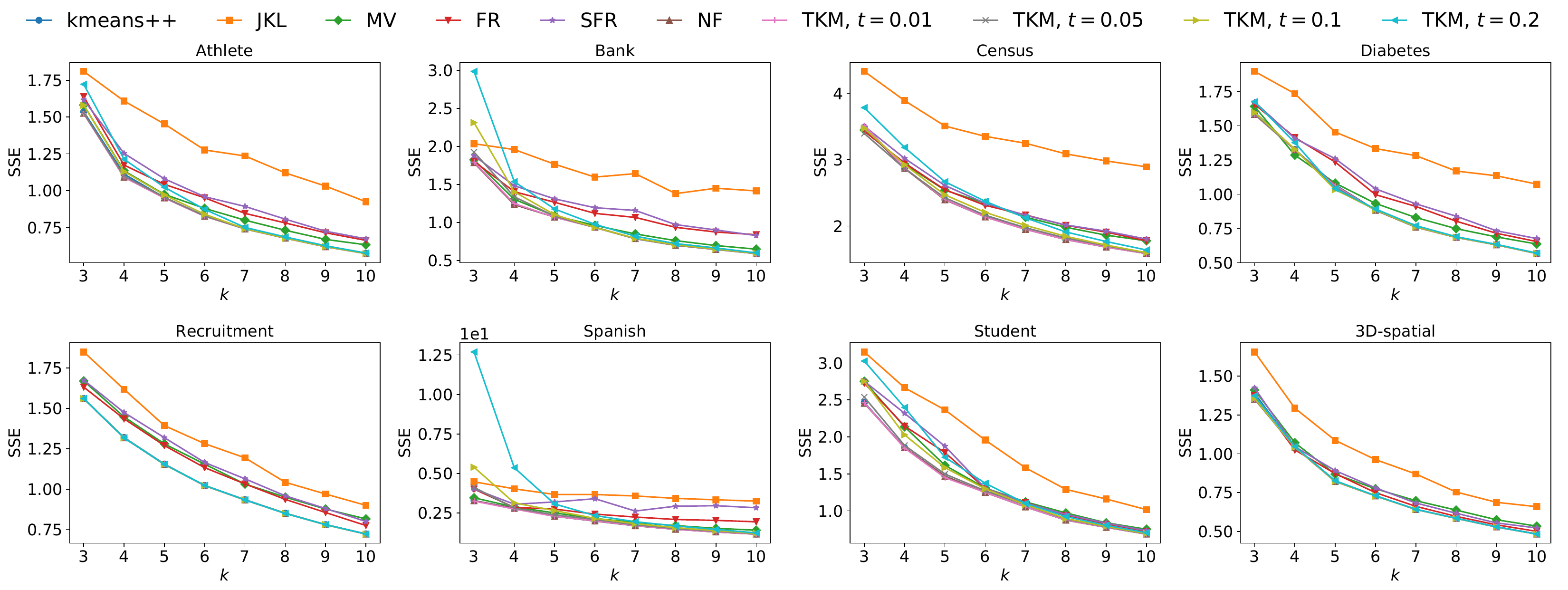}
  \vspace{-2em}
  \caption{Comparison among various methods in terms of SSE at varying values of $k$.}  \label{fig: sse_k}
  \vspace{1em}
  \includegraphics[width=\linewidth]{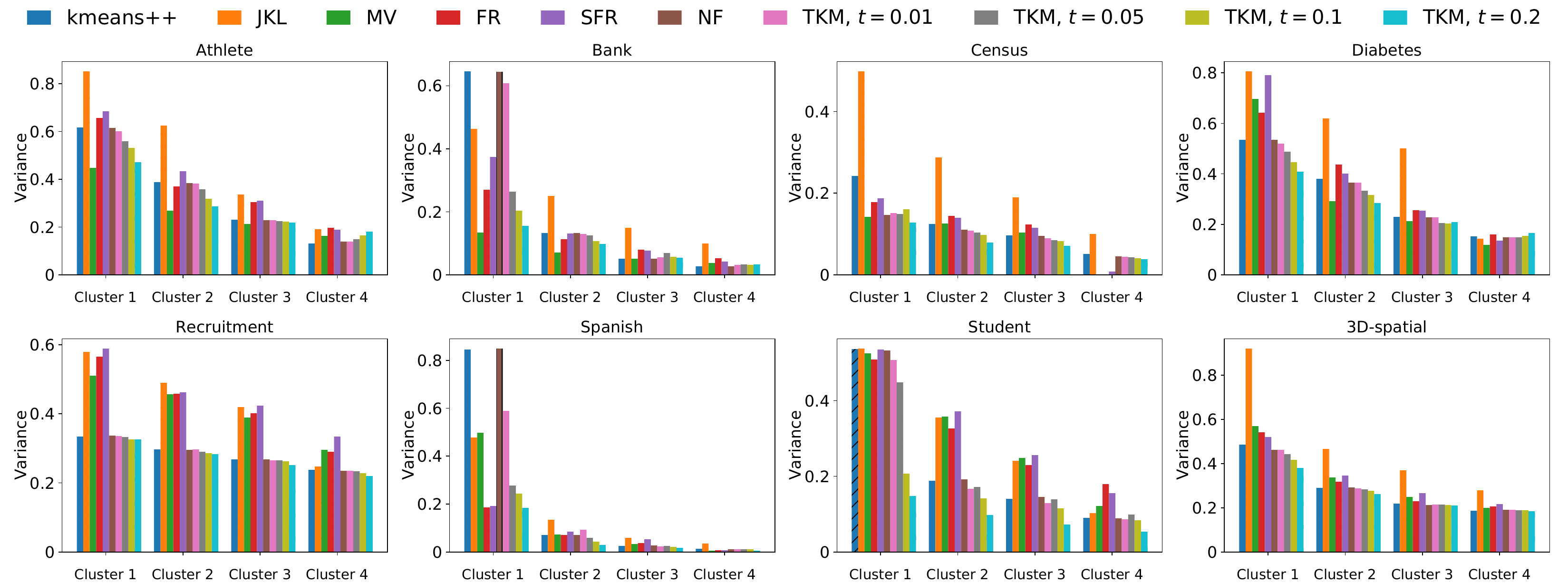}
  \vspace{-2em}
  \caption{Comparison among various methods in terms of the variance in each cluster.}
    \vspace{1em}
  \label{fig: variance}
  \includegraphics[width=\linewidth]{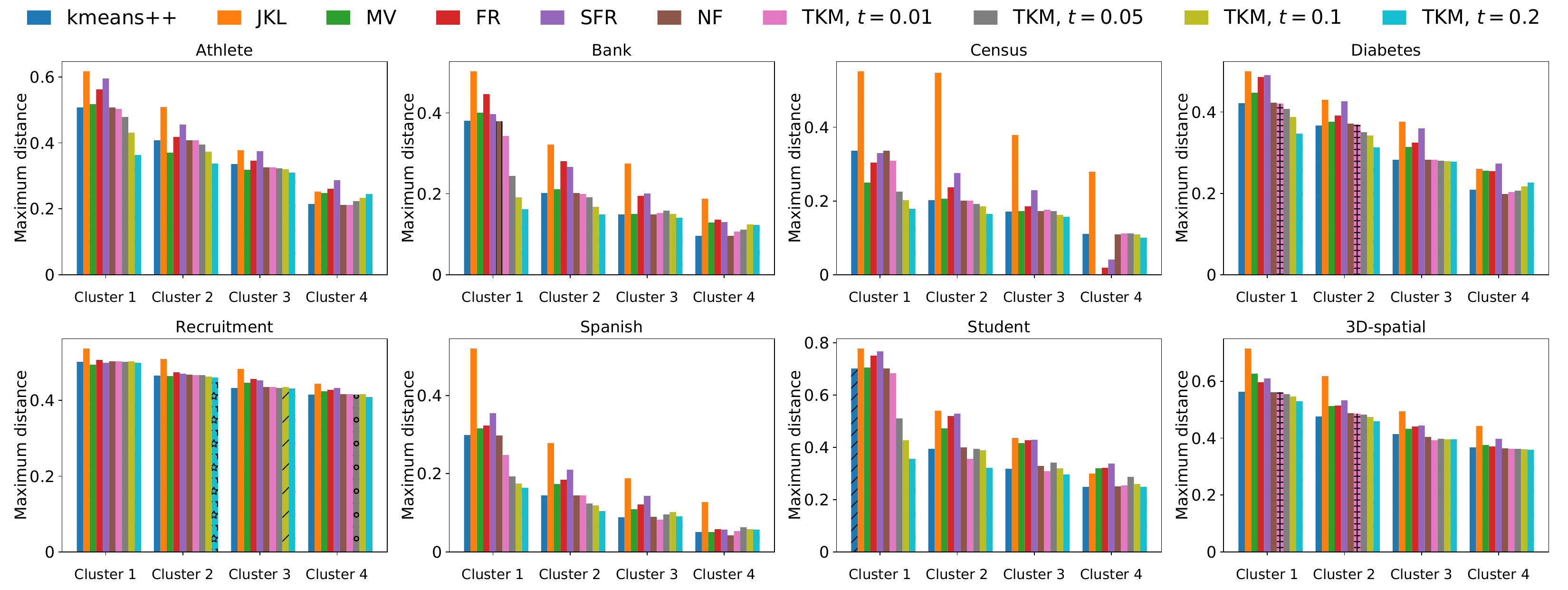}
  \vspace{-2em}
  \caption{Comparison among various methods in terms of the maximum distance in each cluster.}
  \label{fig: largest distance}
\end{figure*}

\begin{table*}[t]
\centering
\caption{Comparison among \textsf{TKM}, \textsf{MV}, and \textsf{FR} in terms of running time (seconds). We abbreviate TLE as the time limit exceeded for 1 hour,  SLE as the sampling size limit exceeded for the dataset dimension, and ROF as the RAM overflow.}
\vspace{-0.51em}
\label{running time}
\renewcommand{\arraystretch}{1.2}
\setlength\tabcolsep{3.2pt}
\begin{tabular}{ccccccccccccccccccccc}
\toprule
\textbf{Dataset}            & \textbf{Method} & & \textbf{1K}  & \textbf{2K} & \textbf{5K} &\textbf{10K} &\textbf{15K}  & \textbf{20K} &\textbf{25K} &\textbf{30K}  &\textbf{40K} &\textbf{50K} &\textbf{60K} &\textbf{70K} &\textbf{80K}&\textbf{90K} & \textbf{2M}     & \textbf{5M}     \\ \midrule
\multirow{4}{*}{\textsf{Census1990}} & \tkm &    & 0.7 & 1.3 & 2.4 & 4.0 &9.0  &12.0&15.9 &18.9 &23.4 & 30.7 &41.5&49.4&56.2&66.0 & 542.3 & SLE    \\
                            & \sfr & & 0.5 & 2.6 & 11.7 & 31.4 &45.4 & 65.2 &77.3 &98.9 &156.7 &195.2&291.8 &483.1 &601.3&ROF& ROF    & SLE   \\ 
                            & \mv & & 5.6 & 30.7 & 85.9 & 250.9&1068.3  & 1783.9&4960.8 &TLE &TLE &TLE  &TLE&TLE&TLE&ROF & ROF    & SLE    \\ 
                            & \fr & & 13.4 & 129.8& 1053.4 & 10692.7  &TLE &TLE &TLE&TLE &TLE &TLE &TLE&TLE&TLE&ROF  & ROF    & SLE    \\ 
                            \midrule
\multirow{4}{*}{\textsf{HMDA}}       & \tkm &    & 0.3 & 1.0 & 2.2 & 3.8&9.3  & 12.3 &15.5 &19.4 &24.3 & 31.1 &45.2 &53.0 &59.1 &71.3& 743.9   & 1901.6\\
& \sfr & & 2.3 & 5.7 & 17.9 & 41.1&65.8 & 72.9  &88.6 &111.8 &174.5 &211.9 &348.8 &528.1 &712.5 &ROF  &ROF   & ROF    \\ 
                            & \mv & & 5.0  & 27.8 & 61.2 & 304.5&406.8 & 1923.9  &5187.6 &TLE &TLE & TLE &TLE&TLE&TLE&ROF & ROF & ROF   \\ 
                            & \fr & & 49.8 & 263.3 & 2784.1 & TLE&TLE & TLE &TLE &TLE &TLE  & TLE&TLE &TLE &TLE&ROF&ROF  & ROF    \\ 
                            \bottomrule
\end{tabular}
\end{table*}

\myparagraph{Measurements} We employ several metrics to evaluate the performance of clustering algorithms.
We use \underline{\textit{SSE}} to measure the utility of different clustering algorithms, where a smaller SSE is considered a better clustering utility. 
To measure fairness among different clustering algorithms, we consider using two metrics. The first is the \underline{\textit{variance}} of each point's distance to its nearest centroid within a cluster. A smaller variance indicates a fairer algorithm. The second metric is the \underline{\textit{maximum distance}} from each point in a cluster to the centroid, where a smaller maximum distance signifies greater fairness.
As for efficiency evaluation, we measure it using the \underline{\textit{running time}} of each algorithm.
To verify the impact of different hyperparameters on the convergence of \tkm, we use \underline{\textit{tilted SSE}} as the metric.

\myparagraph{Implementations} Our algorithms were executed on a platform comprising an Intel i9-14900KF CPU with 24 cores, 64 GB of RAM, and operating on the CentOS 7 environment. 
The software implementations, including our methods and the comparison methods, were realized in Python 3.7 and open-sourced (\url{https://github.com/zsk66/TKM-master}).



\subsection{Comparison among Various Methods}
\subsubsection{Effectiveness Analysis}\label{sec: effectiveness analysis}
Fig. \ref{fig: sse_k} compares the SSE of six methods as $k$ varies on eight datasets:  \athlete, \bank, \census, \diabetes, \recruitment, \spanish, \student, and \spatial. Due to the long running time required by our comparison methods, we need to sample the datasets to accommodate them. We sampled 1000 data points from each dataset, repeated this process 10 times, conducted experiments on the resulting 10 sampled datasets, and averaged the obtained SSE values. We set the parameter $t$ of \tkm to be 0.01, 0.05, 0.1, and 0.2, respectively. The learning rate for \nf and \tkm was set to 0.05, the number of epochs was set to 5, the batch size was set to 100, and the number of iterations was set to 500. \jkl, \mv, \fr, and \sfr adopted the default hyperparameter settings in their papers.

\underline{\textit{Observations}}. 
We can see that as $t$ increases, the SSE of \tkm also increases. This is because an increase in $t$ inevitably brings the centroids closer to the minority data points, resulting in an increase in SSE. 
Comparing the SSE of different methods, we can observe that the SSE of \jkl is consistently the highest across all datasets except for \textsf{Bank} and \textsf{Spanish}.
In these two datasets, \tkm has a large SSE at $t=0.2$, which is due to excessively large $t$ causing the centroids obtained by \tkm to be too close to those minority data points.
The SSE of \sfr is always larger than \fr because \sfr is a version of \fr that applies the sparsification technique.
The SSE for \textsf{3D-spatial} and \textsf{Recruitment} in \fr\ is lower than in \mv, but on the other six datasets, \mv\ has a lower SSE compared to \fr.
Meanwhile, \tkm's SSE at $t=0.01,0.05,0.1$ is consistently lower than \jkl, \mv, \fr, and \sfr, and even performs nearly as well as \kplus and \nf on the \census and \recruitment, which reflects the outstanding effectiveness of \tkm.

\subsubsection{Fairness Analysis}
Fig. \ref{fig: variance} and Fig. \ref{fig: largest distance} illustrate the variance and maximum distance within each cluster for various methods when $k=4$. 
The variance and maximum distance values within Clusters 1-4 are arranged in descending order. 
The data processing and hyperparameter configurations for all methods remain consistent with those outlined in Section \ref{sec: effectiveness analysis}.

\underline{\textit{Observations}}.
From Fig. \ref{fig: variance}, it can be seen that for \tkm, as $t$ increases, the variance of each cluster decreases, which is consistent with our theoretical results. Next, without loss of generality, we examine the variance of each method on Cluster 1. It can be observed that \jkl has the largest variance across all datasets except for \bank, \recruitment, and \spanish, while \kplus and \nf have the largest variance on \bank and \spanish, and \sfr has the largest variance on \recruitment. It is worth noting that in some datasets, such as \diabetes, \recruitment, \student, and \spatial, even when $t=0.01$, the variance of \tkm is smaller than other comparison methods. Moreover, in other datasets, by adjusting $t$, it is always possible to make the variance of \tkm smaller than the comparison methods. 
From Fig. \ref{fig: largest distance}, we observe that the maximum distance within each cluster decreases as $t$ increases. 
This occurs because the greater maximum distance is caused by the centroids being farther from the minority points. With a higher $t$, the centroids shift towards the minority points, thereby reducing the maximum distance. 
Comparing \tkm with other methods reveals that \tkm achieves the smallest maximum distance, demonstrating its fairness.
Moreover, we observe that in \spatial, the variance and maximum distance of \jkl, \mv, \fr, and \sfr are all larger than those of \kplus, indicating that existing individually fair clustering methods might even exacerbate unfairness in our scenario.

\subsubsection{Effeciency Analysis}
Table \ref{running time} presents a comparison of the running time of \tkm with three state-of-the-art methods, \mv, \fr, and \sfr (
Due to the poor performance of \jkl and \nf in effectiveness and fairness, we do not consider these two methods in the comparison of efficiency). We sampled the \censusbig and \hmda with sizes $n_s$ of 1K, 2K, 5K, 10K,15K, 20K, 25K, 30K, 40K, 50K, 60K, 70K, 80K, 90K, 2M, and 5M, respectively. We set the number of iterations for \tkm to 500, the batch size to $\frac{1}{50}n_s$, the number of epochs to 5, and the learning rate to 0.05. The hyperparameters for \mv, \fr, and \sfr were set to their default values in their papers.

\begin{figure*}[p]
  \centering
  \includegraphics[width=\linewidth]{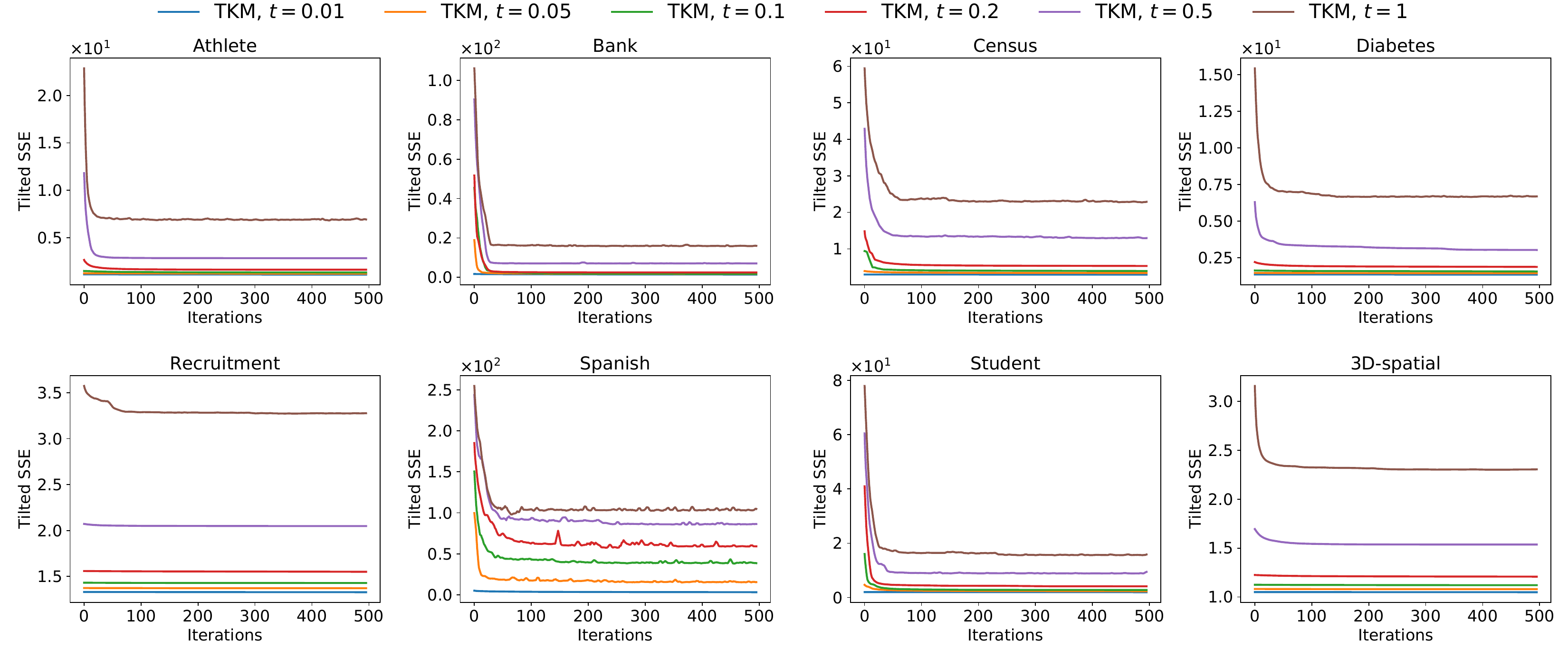}
  \vspace{-2em}
  \caption{Effect of $t$ on the convergence of \tkm.}
    \vspace{1em}
  \label{fig: tilted sse with t}
    \includegraphics[width=\linewidth]{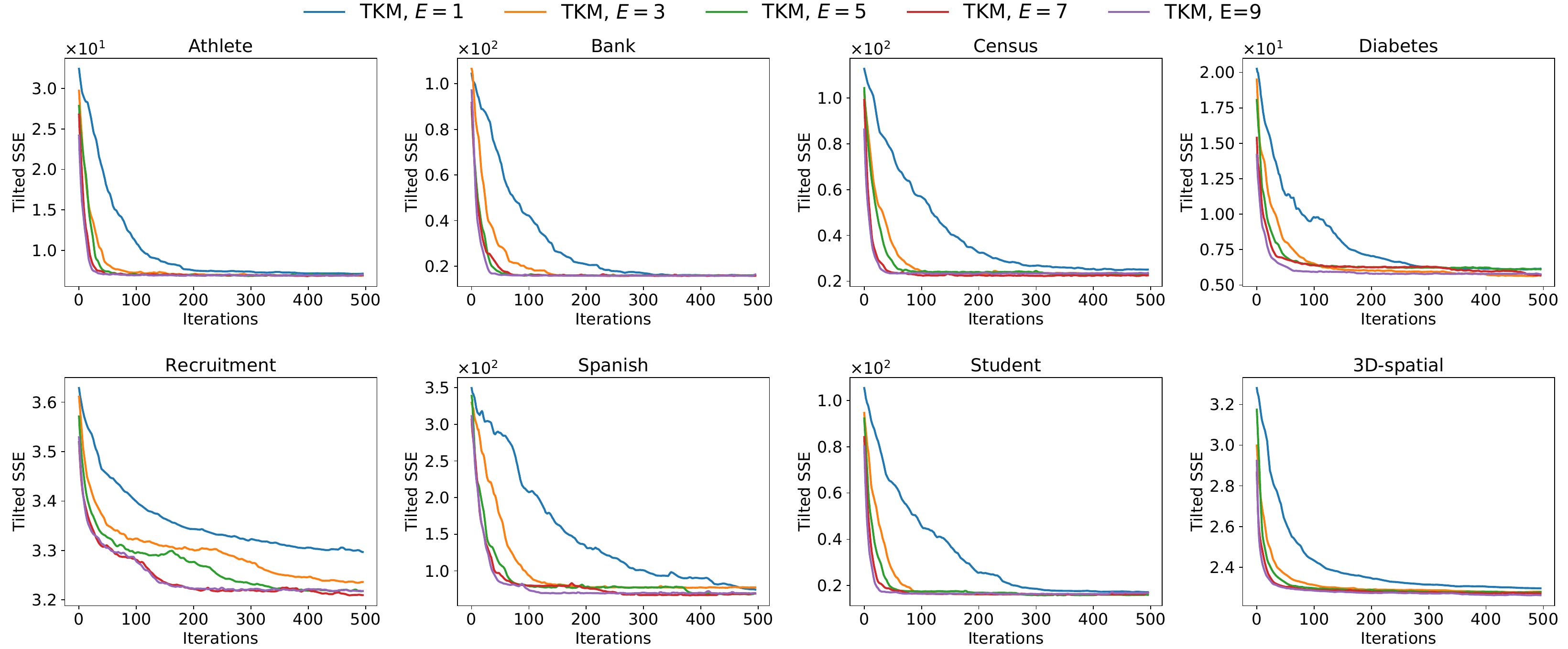}
  \vspace{-2em}
  \caption{Effect of the epoch on the convergence of \tkm.}
    \vspace{1em}
  \label{fig: tilted sse with epoch}
    \includegraphics[width=\linewidth]{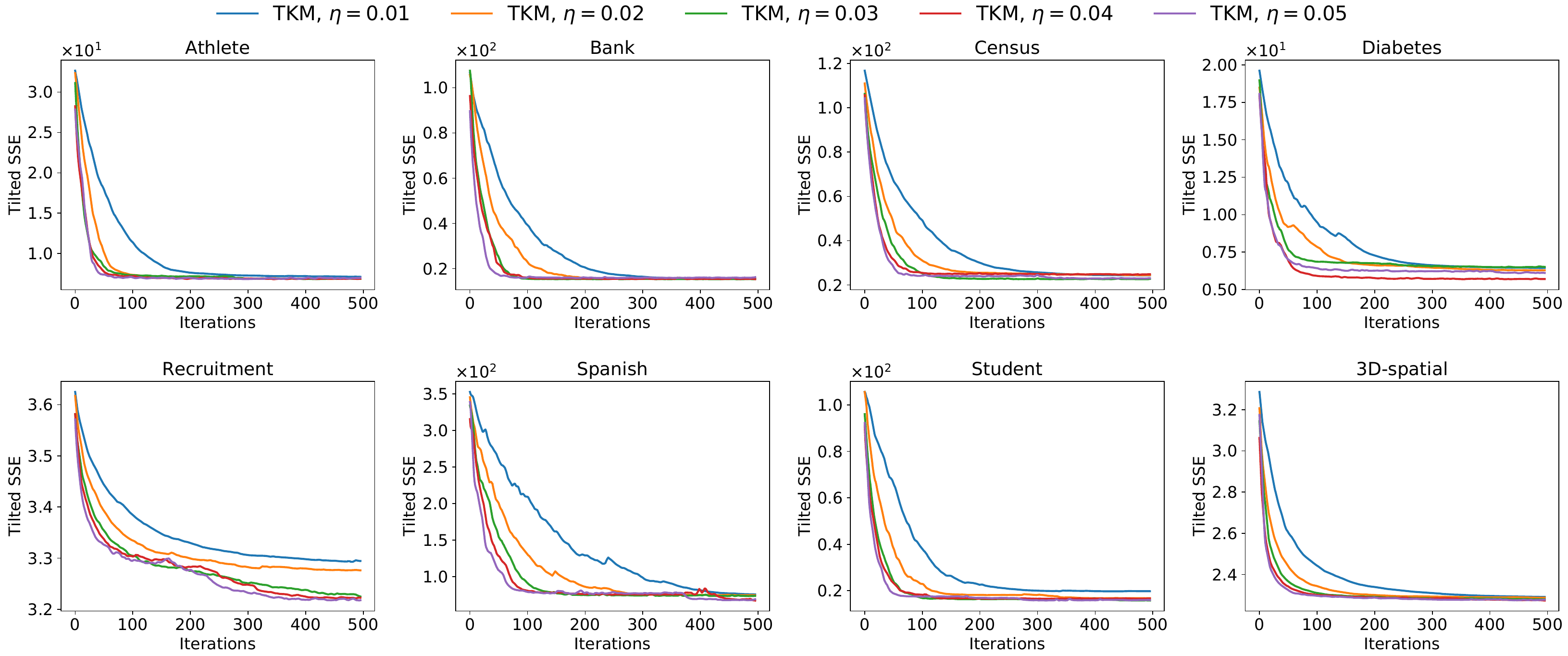}
  \vspace{-2em}
  \caption{Effect of the learning rate on the convergence of \tkm.}
  \vspace{1em}
  \label{fig: tilted sse with lr}
\end{figure*}



\begin{figure*}[t]
  \centering
  \includegraphics[width=\linewidth]{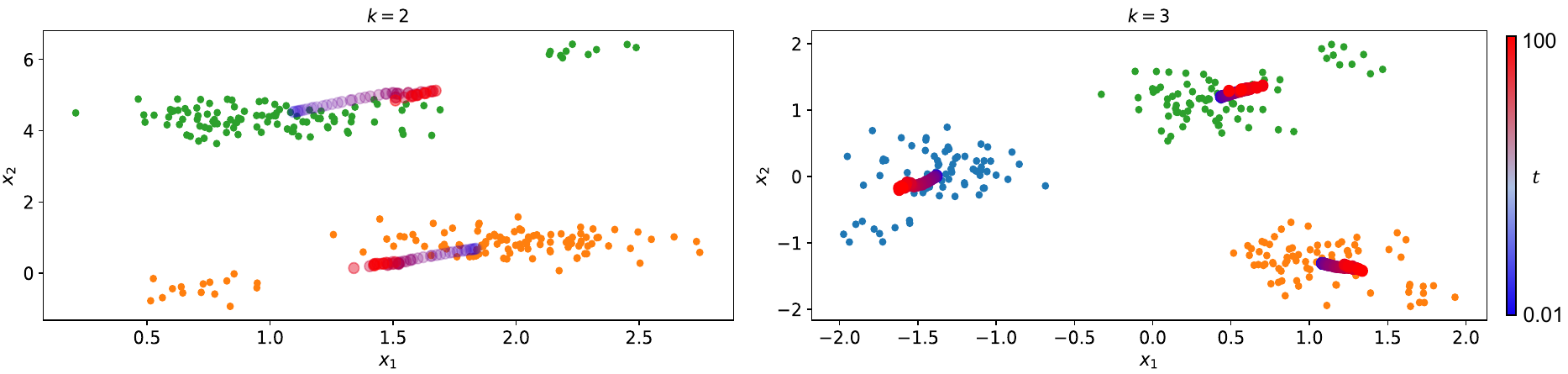}
  \vspace{-2em}
  \caption{Visualization of two synthetic 2-dimensional data for $k=2$ and $k=3$ by \tkm.}
  \vspace{-1em}
  \label{fig: visualization}
\end{figure*}

\underline{\textit{Observations}}. Experimental results demonstrate that regardless of the number of data points sampled, the running time of \tkm is always significantly shorter than that of \mv, \fr, and \sfr.
It can be observed that \tkm can cluster 5 million data points in about 30 minutes, while \mv can only cluster 20,000 samples within 30 minutes, and \fr can only cluster 5,000 data points. Moreover, it is worth noting that as the sample size increases, \tkm's running time increases by hundreds or even thousands of times compared to \mv and \fr. For example, when the number of sampled points is 1,000, \tkm achieves 8.0$\times$ and 19.1$\times$ acceleration compared to \mv and \fr in \censusbig, respectively. When the number of sampled points is 10,000, \tkm achieves 62.7$\times$ and 2673.2$\times$ acceleration compared to \mv and \fr in \censusbig, respectively. Furthermore, although the running time of \sfr is significantly shorter than \mv and \fr, \tkm still achieves approximately a 10.7$\times$ and 12.1$\times$ acceleration with 80,000 data points in \censusbig and \hmda, respectively. Furthermore, when the sample size reaches 90,000, the algorithmic characteristic of \sfr, which requires computing distances between each sample point, can lead to a memory overflow issue, causing the algorithm to terminate. This issue also arises in \mv and \fr.

\subsubsection{Summary of Lessons Learned}
We have provided the changes of SSE with respect to $k$ for different methods, the variance results in four clusters for different methods, and a comparison of the efficiency of different methods. Our experimental results have led us to draw the following conclusions:
\begin{itemize}
    \item \tkm outperforms state-of-the-art methods in terms of effectiveness. Specifically, \tkm achieves smaller SSE compared to state-of-the-art methods across different values of $k$ and $t$. In some datasets, the SSE of \tkm is almost the same as methods that do not consider individual fairness.
    \item \tkm outperforms state-of-the-art methods in terms of fairness. Specifically, \tkm can achieve smaller variance and maximum distance than state-of-the-art methods when an appropriate value of $t$ is chosen.
    \item \tkm surpasses state-of-the-art methods in terms of efficiency. Specifically, \tkm can cluster more data points in a shorter time, and as the sample size increases, this acceleration effect becomes even more pronounced. 
    Moreover, \tkm can overcome the RAM overflow issue that existing methods encounter when dealing with large-scale data.
\end{itemize}

\subsection{Comparison among Various Parameters}
\subsubsection{Tilted SSE vs. $t$}\label{sec: sse with t}
Fig. \ref{fig: tilted sse with t} illustrates the convergence of tilted SSE with iterations at $t$ values of 0.01, 0.05, 0.1, 0.2, 0.5, and 1.
We randomly select 1000 data points from each dataset, repeating this process 10 times. We then conduct experiments on these 10 subsampled datasets, calculating the average of the resulting tilted SSE values.
For other hyperparameters, we set the learning rate to 0.05, the number of iterations to 500, the batch size to 100, and the epoch size to 5.

\underline{\textit{Observations}}.
We observe that despite using SGD to update the centroids, the tilted SSE of \tkm still decreases steadily with iterations, which confirms the convergence of \tkm.
As $t$ increases, the tilted SSE also increases. This confirms that our theoretical analysis of the monotonicity of the tilted SSE with respect to $t$ holds not only for $k=1$.
When $t=0.01$, the tilted SSE remains nearly unchanged with iterations. This indicates that the tilted SSE is insensitive to variations in $t$ when $t$ is small.

\subsubsection{Tilted SSE vs. Epoch}\label{sec: sse with epoch}
Fig. \ref{fig: tilted sse with epoch} illustrates the convergence of tilted SSE with different numbers of epochs during iterations.
The data preprocessing for \tkm here follows the same procedures outlined in Section \ref{sec: sse with t}.
To visualize the curve of tilted SSE of \tkm over iterations more intuitively, we set $t=0.5$, learning rate to 0.03, number of iterations to 500, batch size to 50, and epoch size to 1, 3, 5, 7, 9.

\underline{\textit{Observations}}.
From Fig. \ref{fig: tilted sse with epoch}, it can be observed that as the number of iterations increases, the tilted SSE of \tkm decreases and tends to stabilize after reaching a certain value on all datasets. 
With an increase in the epoch size, the convergence speed of \tkm accelerates, and its convergence performance improves. This is because increasing the epoch size allows for higher precision in the solution obtained through SGD during each iteration, as more data can be utilized. 
When the epoch size is 7 and 9, the convergence and convergence speed of \tkm are not significantly different. Therefore, selecting 7 as the epoch size is an appropriate choice.
However, in some datasets, we found that increasing the epoch size does not necessarily improve convergence. For example, in \texttt{Recruitment}, a smaller epoch size of 7 yields better convergence compared to an epoch size of 9. This is attributed to the risk of overfitting when the epoch size is too large. Therefore, choosing an epoch size of 7 is deemed appropriate for these datasets.

\subsubsection{Tilted SSE vs. Learning Rate}
Fig. \ref{fig: tilted sse with lr} illustrates the convergence of tilted SSE with various learning rates during iterations. The data preprocessing for \tkm here is the same as in Section \ref{sec: sse with t}. 
For the parameter settings of \tkm, we set $t=0.5$, epoch size as 5, batch size as 50, number of iterations to 500, and learning rate as 0.01, 0.02, 0.03, 0.04, and 0.05.

\underline{\textit{Observations}}.
From Fig. \ref{fig: tilted sse with lr}, we can see that, across the eight datasets, the convergence speed generally increases with the increase in learning rate. However, when the learning rate increases to a certain extent, the increase in convergence speed becomes slower. For example, when $\eta=0.03,0.04,0.05$, the convergence speed and the converged tilted SSE value on the \texttt{Bank} are almost indistinguishable. Additionally, if the learning rate is excessively high, it can result in poorer convergence, as demonstrated in \texttt{Diabetes} where $\eta=0.04$ produces a smaller tilted SSE. This occurs because an overly large learning rate may cause the SGD step size to become excessive, hindering the achievement of locally optimal solutions.

\subsubsection{Visualization}
Fig. \ref{fig: visualization} demonstrates how centroids change over $t$ in two synthetic datasets when the number of clusters is set to 2 and 3, respectively.
We set the number of epochs for \tkm to 5, the number of iterations to 1000, and the learning rate to 0.01, the batch size to 20.
For the values of $t$, we take a total of 60 geometrically spaced values between $10^{-2}$ and $10^2$.
We employ a blue-to-red gradient to depict the rising values of $t$, and we use the same color to represent data points within the same cluster.

\underline{\textit{Observations}}.
It can be observed that as $t$ increases, the positions of the centroids tend to shift towards the minority data points in each cluster. 
This ensures data points in each cluster can guarantee ``treat all points equally'', aligning with the concept of individual fairness.
Furthermore, we observe that as $t$ increases, the centroids do not shift excessively towards minority data points, ensuring that the distance from majority data points to the centroids remains reasonable. This demonstrates that \tkm ensures equal treatment of each data point.

\subsubsection{Summary of Lessons Learned}
We have provided the convergence behavior of \tkm under different epoch sizes and learning rates, as well as visualizations of \tkm on 2-dimension synthetic data. These experiments lead us to the following conclusions:
\begin{itemize}
    \item \tkm is a convergent algorithm, and the tilted SSE increases monotonically with $t$. Specifically, for different values of $t$, the tilted SSE in \tkm steadily decreases to a stable value. Moreover, as $t$ increases, the tilted SSE increases.
    \item The convergence of \tkm is influenced by the epoch size and learning rate. Specifically, selecting an appropriate epoch size and learning rate can lead to faster convergence speed and better convergence of \tkm. However, choosing larger epoch sizes and learning rates does not necessarily improve the performance.
    \item \tkm indeed can ensure individual fairness for $k$-means. Specifically, as $t$ increases, it can guarantee that those minority data points can be closer to the centroids, achieving the goal of treating each individual equally. 
\end{itemize}

\section{Conclusions and Future Work}\label{sec conclusion}
This paper investigated the individually fair $k$-means in the context of location-based resource allocation.
To address the issue where existing individually fair clustering methods and fairness metrics may exacerbate unfairness, we proposed \tkm, an algorithm designed to effectively solve the individually fair $k$-means problem via exponential tilting.
We constructed the tilted SSE as the objective function and proposed solving the optimization problem using CD and SGD. Moreover, we proposed to employ variance to measure fairness.
Our theory and experiments have validated that the effectiveness, efficiency, and fairness of our proposed algorithm outperform existing state-of-the-art methods.
It is noteworthy that existing individually fair clustering methods encounter challenges in their application to large-scale data clustering analysis scenarios, primarily due to their computational complexity, which depends on the dataset size. In contrast, \tkm, due to its excellent efficiency performance, can be applied in many big data clustering analysis scenarios, such as resource allocation.

Due to privacy concerns, data is often stored on different devices and cannot be shared among them. Therefore, a hot topic of research is how to perform clustering analysis without sharing data. In the future, we will investigate individually fair $k$-means in the framework of federated learning to address this issue.

\section{Proofs}\label{sec proof}

\subsection{Proof of Theorem \ref{theorem 1}}\label{sec: proof of theorem 1}
Before proving Theorem \ref{theorem 1}, we present some useful lemmas.

\begin{lemma}\label{lemma: increasing psi pi phi}
Given a cluster $\mathcal{S}_j$, let $\bdelta_j$ and $\bc_j$ be the corresponding assignment and centroid, then for any $t\geq0$, it holds that,
\begin{align}
    \psi(\bdelta_j,\bc_j)\leq\phi(t,\bdelta_j,\bc_j).
\end{align}
\end{lemma}
\begin{proof}
    Following from \eqref{eq: tkm}, we have
\begin{align}
    \phi(t,\bdelta_j,\bc_j)&=\frac{1}{t}\log\frac{1}{n}\sum_{i=1}^{n}e^{tf(\bx_i,\bc_j)\delta_{ij}}\notag\\
    &\geq\frac{1}{t}\cdot\frac{1}{n}\sum_{i=1}^{n}\log e^{tf(\bx_i,\bc_j)\delta_{ij}}\label{eq: jensen implies 32}\\
    &=\frac{1}{n}\sum_{i=1}^{n} f(\bx_i,\bc_j)\delta_{ij}=\psi(\bdelta_j,\bc_j),
\end{align}
where \eqref{eq: jensen implies 32} follows from the Jensen's inequality. 
\end{proof}


\begin{lemma}\label{lemma: epsilon geq}
    Given a set of clusters $\mathcal{S}=\{\mathcal{S}_j\}_{j=1}^k$ and a set of centroids $\mathcal{C}=\{\bc_j\}_{j=1}^k$, let $dist(\bx_i,\mathcal{C}):=\min_{\bc_j\in\mathcal{C}}\|\bx_i-\bc_j\|^2$,  then for any $t\geq0$, there exists a scalar $\epsilon\geq k\cdot \frac{\max_{\bx_i\in X}dist(\bx_i,\mathcal{C})}{\min_{\bx_i\in X}dist(\bx_i,\mathcal{C})}$, such that the following inequality holds:
    \begin{align}
        \overline{\phi}(t,\mathcal{S},\mathcal{C})\leq\epsilon \cdot \overline{\psi}(t,\mathcal{S},\mathcal{C}).
    \end{align}
\end{lemma}
\begin{proof}
    Consider the case when $t\to\infty$, according to L'Hôpital's rule, it holds that,
\begin{align}\label{eq: 39 lim to infty}
    \lim_{t\to\infty}\phi(t,\bdelta_j,\bc_j)=\max_{\bx_i\in\mathcal{S}_j}f(\bx_i,\bc_j),
\end{align}
which implies that for any $j\in[k]$, $\phi(t,\bdelta_j,\bc_j)$ is bounded. Then there must exist a scalar $\epsilon$ such that
\begin{align}
    \epsilon&\geq k \cdot \frac{\max_{\bx_i\in X}dist(\bx_i,\mathcal{C})}{\min_{\bx_i\in X}dist(\bx_i,\mathcal{C})}\notag \\
    &=\frac{\sum_{j=1}^k \frac{1}{t}\log\frac{1}{n}\sum_{i=1}^{n}e^{t\max_{\bx_i\in X}dist(\bx_i,\mathcal{C})}}{\min_{\bx_i\in X}dist(\bx_i,\mathcal{C})}\\
    &\geq \frac{\sum_{j=1}^k  \frac{1}{t}\log\frac{1}{n}\sum_{i=1}^{n}e^{tf(\bx_i,\bc_j)\delta_{ij}}}{\frac{1}{n}\sum_{j=1}^k  \sum_{\bx_i\in\mathcal{S}_j}f(\bx_i,\bc_j)}\label{eq: 40 follows from 39 and increasing}\\
    &=\frac{\overline{\phi}(t,\mathcal{S},\mathcal{C})}{\overline{\psi}(\mathcal{S},\mathcal{C})}\label{eq: 41 follows from pi geq phi},
\end{align}    
which completes the proof.
\end{proof}



\begin{proposition}\label{proposion: phi geq psi}
    Let $\delta_j^{\star},\bc_j^{\star},j\in[k]$ be the optimal solution of SSE, let $\bdelta_j^*,\bc_j^*,j\in[k]$ be the optimal solutions of tilted SSE, and let $\overline{\psi}^{\star}$, $\overline{\phi}^*$ be the corresponding optimal objective function values, then for any $t\geq 0$, we have $\overline{\psi}^{\star}\leq\overline{\phi}^*$.
\end{proposition}

\begin{proof}
    Based on Lemma \ref{lemma: increasing psi pi phi} and optimal conditions, we obtain
\begin{align}\label{eq: leq 42}
   \overline{\psi}^{\star}= \psi(\bdelta_j^{\star},\bc_j^{\star})\leq\psi(\bdelta_j^{*},\bc_j^{*})\leq\phi(t,\bdelta_j^{*},\bc_j^{*})=\overline{\phi}^*.
\end{align}        

Summing over \eqref{eq: leq 42} from 1 to $k$ implies Proposition \ref{proposion: phi geq psi}.
\end{proof}

\begin{lemma}[Theorem 1.1 in \cite{arthur2007k}]\label{lemma: Theorem 1.1 in arthur}
    Let $\overline{\psi}^{\star}$ be the optimal SSE of $k$-means, let $\mathcal{C}$ be the centroids set constructed by $k$-means++, and let $\mathcal{S}$ be the corresponding induced assignment, then for any set of data points, it holds that $\mathbb{E}[\overline{\psi}(\mathcal{S},\mathcal{C})]\leq8(\log k+2)\overline{\psi}^{\star}$.
    
\end{lemma}

Next, we are ready to prove Theorem \ref{theorem 1} based on the above lemmas.
\begin{proof}[Proof of Theorem \ref{theorem 1}]
    Let $\mathcal{C}$ be the centroids set constructed by $k$-means++, and let $\mathcal{S}$ be the corresponding induced set of clusters, then following from Lemma \ref{lemma: epsilon geq}, we have
\begin{align}
    \overline{\phi}(t,\mathcal{S},\mathcal{C})\leq \epsilon\cdot\overline{\psi}(\mathcal{S},\mathcal{C}),
\end{align}    

Then we can bound $\mathbb{E}[\overline{\phi}(t,\mathcal{S},\mathcal{C})]$ as
\begin{align}
    \mathbb{E}[\overline{\phi}(t,\mathcal{S},\mathcal{C})]&\leq \epsilon\cdot \mathbb{E}[\overline{\psi}(\mathcal{S},\mathcal{C})]\notag\\
    &\leq 8\epsilon(\log k+2)\cdot\overline{\psi}^{\star}\label{eq:41 follows from lemma 5}\\
    &\leq 8\epsilon(\log k+2)\cdot\overline{\phi}^*,\label{eq:42 follows from proposition 1}
\end{align}
where \eqref{eq:41 follows from lemma 5} follows from Lemma \ref{lemma: Theorem 1.1 in arthur}, and \eqref{eq:42 follows from proposition 1} follows from Proposition \ref{proposion: phi geq psi}. According to Lemma \ref{lemma: epsilon geq}, we have $\epsilon\geq k \cdot \frac{\max_{\bx_i\in X}dist(\bx_i,\mathcal{C})}{\min_{\bx_i\in X}dist(\bx_i,\mathcal{C})}$, then we can derive that $\mathbb{E}[\overline{\phi}(t,\mathcal{S},\mathcal{C})]\leq O(k\log k)\overline{\phi}^*$, which completes the proof.
\end{proof}

\subsection{Proof of Theorem \ref{theorem: convergence}}\label{proof of theorem convergence}
By the Mean Value Theorem, the gradient Lipschitz continuity indicates the following proposition.
\begin{proposition}\label{proposition: gradient lipschitz}
    For any $t\geq 0$, and $\tilde{\bc}_j,\tilde{\bc}_j'\in\mathbb{R}^d$, it holds that
\begin{small}
\begin{align*}
    \phi(t,\bdelta_j,\tilde{\bc}_j)&-\phi(t,\bdelta_j,\tilde{\bc}_j')\leq\notag\\
    \nabla_{\bc_j}\phi(t,\bdelta_j,\tilde{\bc}_j')^{\top}&(\tilde{\bc}_j-\tilde{\bc}_j')+\frac{L(t)}{2}\|\tilde{\bc}_j-\tilde{\bc}_j'\|^2.
\end{align*}  
\end{small}
\end{proposition}
\begin{proof}
    Following Lemma \ref{lemma: gradient lipschitz of tilted sse}, it holds that
\begin{align*}
    &\phi(t,\bdelta_j,\tilde{\bc}_j)=\phi(t,\bdelta_j,\tilde{\bc}_j')+\int_{0}^1\frac{\partial\phi(t,\bdelta_j,\tilde{\bc}_j'+y(\tilde{\bc}_j-\tilde{\bc}_j'))}{\partial y}dy\notag\\
    &=\phi(t,\bdelta_j,\tilde{\bc}_j')\!+\!\!\int_{0}^1 \nabla_{\bc_j}\phi(t,\bdelta_j,\tilde{\bc}_j'+y(\tilde{\bc}_j-\tilde{\bc}_j'))^{\top}(\tilde{\bc}_j-\tilde{\bc}_j')dy\notag\\
    &=\phi(t,\bdelta_j,\tilde{\bc}_j')+\nabla_{\bc_j}\phi(t,\bdelta_j,\tilde{\bc}_j')^{\top}(\tilde{\bc}_j-\tilde{\bc}_j')+\notag\\
    &\!\!\int_{0}^1 \![\nabla_{\bc_j}\phi(t,\bdelta_j,\tilde{\bc}_j'\!+\!y(\tilde{\bc}_j-\tilde{\bc}_j'))\!-\!\nabla_{\bc_j}\phi(t,\bdelta_j,\tilde{\bc}_j')]^{\top}(\tilde{\bc}_j-\tilde{\bc}_j')dy\notag\\
    &\leq\phi(t,\bdelta_j,\tilde{\bc}_j')+\nabla_{\bc_j}\phi(t,\bdelta_j,\tilde{\bc}_j')^{\top}(\tilde{\bc}_j-\tilde{\bc}_j')+\frac{L(t)}{2}\|\tilde{\bc}_j-\tilde{\bc}_j'\|^2,
\end{align*}    
which completes the proof.
\end{proof}

Next, we show the proof of Theorem \ref{theorem: convergence}.
\begin{proof}[Proof of Theorem \ref{theorem: convergence}]
    We consider proving the decreasing property of \tkm from two parts: refinement and assignment. Our proof with respect to the refinement follows from \cite{bottou2018optimization} which establishes the convergence for gradient Lipschitz continuous objective functions.
    Under the gradient Lipschitz continuous property of $\phi(t,\bdelta_j,\bc_j)$ with respect to $\bc_j$, the iterations of SGD satisfy the following inequality by applying Proposition \ref{proposition: gradient lipschitz}:
\begin{align}\label{eq: descent sgd}
    &\mathbb{E}[\phi(t,\bdelta_j^{it},\bc_j^{it+1})]-\phi(t,\bdelta_j^{it},\bc_j^{it})\!\leq\!\notag\\
    &\!\!\!\!\!\!\!\!-\eta\nabla_{\bc_j}\phi(t,\bdelta_j^{it},\bc_j^{it})^{\top}\mathbb{E}[g(\mathcal{B},\bc_j^{it})]\!+\!\frac{1}{2}\eta^2 L(t)\mathbb{E}[\|g(\mathcal{B},\bc_j)\|^2].\!\!\!
\end{align}    

According to Cauchy-Schwarz inequality and Assumption \ref{assumption: expetation}, it holds that
\begin{align}
    \|\mathbb{E}[g(\mathcal{B},\bc_j^{it})]\|^2\geq\mu^2\|\nabla_{\bc_j}\phi(t,\bdelta_j^{it},\bc_j^{it})\|^2.
\end{align}

Next, we consider bounding $\mathbb{E}[\|g(\mathcal{B},\bc_j^{it})\|^2]$ under Assumption \ref{assumption: expetation} as follows,
\begin{align}\label{eq: square bound}
    \!\!\mathbb{E}[\|g(\mathcal{B},\bc_j^{it})\|^2] \!&=\! \mathbb{E}[\|g(\mathcal{B},\bc_j^{it})\!-\!\mathbb{E}[g(\mathcal{B},\bc_j^{it})]\|^2]\!+\!\|\mathbb{E}[g(\mathcal{B},\bc_j^{it})]\|^2\notag\\
    &\leq \nu+\nu_G\cdot\|\nabla_{\bc_j}\phi(t,\bdelta_j,\bc_j^{it})\|^2,
\end{align}
where $\nu_G:=\nu_H+\mu_G^2\geq\mu^2$. Then by applying Assumption \ref{assumption: expetation} and \eqref{eq: square bound} into \eqref{eq: descent sgd}, we obtain
\begin{align}\label{eq: refinement inequality}
    \mathbb{E}&[\phi(t,\bdelta_j^{it},\bc_j^{it+1})]-\phi(t,\bdelta_j^{it},\bc_j^{it})\leq\notag\\
   & -\bigl(\mu-\frac{1}{2}\eta\nu_G L(t)\bigr)\eta\|\nabla_{\bc_j}\phi(t,\bdelta_j,\bc_j)\|^2+\frac{1}{2}\eta^2\nu L(t).
\end{align}

To ensure that the objective function value decreases within refinement, we need $\mu-\frac{1}{2}\eta\nu_G L(t)>0$,
which implies $\eta<\frac{2\mu}{\nu_G L(t)}\leq\frac{2}{\mu L(t)}$. Next, we consider proving the decreasing property in the assignment. Following the optimal condition with $\bdelta_j$, the following inequality holds
\begin{align}\label{eq: assignment inequality}
    \mathbb{E}[\phi(t,&\bdelta_j^{it+1},\bc_j^{it+1})]\leq\mathbb{E}[\phi(t,\bdelta_j^{it},\bc_j^{it+1})].
\end{align}

Combining \eqref{eq: refinement inequality} and \eqref{eq: assignment inequality} yields
\begin{align}\label{eq: phi bound it it+1}
    \mathbb{E}[\phi(t,&\bdelta_j^{it+1},\bc_j^{it+1})]\leq\phi(t,\bdelta_j^{it},\bc_j^{it}).
\end{align}

Summing over \eqref{eq: phi bound it it+1} from $1$ to $k$ proves Theorem \ref{theorem: convergence}.
\end{proof}

\subsection{Proof of Theorem \ref{theorem: variance decrease with t}}\label{sec: proof of theorem fairness}

We begin by defining the tilted weight, tilted empirical mean, and tilted empirical variance when all data points are normalized to a unit norm.

\begin{definition}[Tilted gradient and weight]
    Suppose the dataset is normalized, then the tilted weight is defined as 
    \begin{align*}
        w_i(t,\mathcal{S}_j,\bc_j)\!:=\!\frac{e^{t\|\bx_i-c_j\|^2}}{\sum_{\bx_i\in\mathcal{S}_j} e^{t\|\bx_i-\bc_j\|^2}}\!=\!\frac{1}{|\mathcal{S}_j|}e^{-2t\bc_j^{\top}\bx_i-\Gamma(t,\mathcal{S}_j,\bc_j)},\!\!
    \end{align*}        
    where $\Gamma(t,\mathcal{S}_j,\bc_j):=\log\frac{1}{|\mathcal{S}_j|}\sum_{\bx_i\in\mathcal{S}_j} e^{-2t\bc_j^{\top}\bx_i}$.

\end{definition}


\begin{definition}[Tilted empirical mean and variance] \label{definition: tilted mean and variance under assumption 1}
    Suppose the dataset is normalized, the tilted empirical mean and variance in each cluster are defined as
\begin{small}
\begin{align*}
    &\mathbb{E}_t\Bigl(\!\mathrm{f}(\mathcal{S}_j,\bc_j)\!\Bigr)\!:=\!\|\bc_j\|^2+\!\!\!\!\!\sum_{\bx_i\in\mathcal{S}_j}\!\!\!\!w_i(t,\mathcal{S}_j,\bc_j)\|\bx_i\|^2-\bc_j^{\top}M(t,\mathcal{S}_j,\bc_j),\!\\
    &\!\!\!\!\var_t\Bigl(\!\mathrm{f}(\mathcal{S}_j,\bc_j)\!\Bigr)\!:=\!\mathbb{E}_t\Bigl(\bc_j^{\top}\bigl(-2\bx_i\!-\!M(t,\mathcal{S}_j,\bc_j)\bigr)\!\Bigr)^2\!\!\!=\!\bc_j^{\top} V(t,\mathcal{S}_j,\bc_j)  \bc_j,
\end{align*} 
\end{small}   
    where $M(t,\mathcal{S}_j,\bc_j):=\sum_{\bx_i\in\mathcal{S}_j}2w_i(t,\mathcal{S}_j,\bc_j)\bx_i$, and 
\begin{small}
\begin{align*}
    &V(t,\mathcal{S}_j,\bc_j)\!:=\mathbb{E}_t\bigl(-2\bx_i\!-\!M(t,\mathcal{S}_j,\bc_j)\bigr)^{\top}\bigl(-2\bx_i-M(t,\mathcal{S}_j,\bc_j)\bigr)\notag\\
    &=\!\!\!\sum_{\bx_i\in\mathcal{S}_j}\!\!w_i(t,\mathcal{S}_j,\bc_j)\bigl(-2\bx_i-M(t,\mathcal{S}_j,\bc_j)\bigr)^{\top}\bigl(-2\bx_i-M(t,\mathcal{S}_j,\bc_j)\bigr).
\end{align*}      
\end{small}

\end{definition}

\begin{lemma}[Partial derivatives of $M(t,\mathcal{S}_j,\bc_j)$ and $\Gamma(t,\mathcal{S}_j,\bc_j)$]\label{lemma:Partial derivatives}
    For any $t\geq 0$, and any $\bc_j\in\mathbb{R}^d$, it holds that
    \begin{align}
        \frac{\partial}{\partial t}M(t,\mathcal{S}_j,\bc_j)=-V(t,\mathcal{S}_j,\bc_j)\bc_j,\label{eq25}\\
        \nabla_{\bc_j} M(t,\mathcal{S}_j,\bc_j)=-tV(t,\mathcal{S}_j,\bc_j),\label{eq26}\\
        \frac{\partial}{\partial t}\Gamma(t,\mathcal{S}_j,\bc_j)=-\bc_j M(t,\mathcal{S}_j,\bc_j),\label{eq27}\\
        \nabla_{\bc_j} \Gamma(t,\mathcal{S}_j,\bc_j)=-t M(t,\mathcal{S}_j,\bc_j)\label{eq28}.
    \end{align}
\end{lemma}

\begin{proof}[Proof of Theorem \ref{theorem: variance decrease with t}]
Let $\bc_j(t):=\tm(t,\mathcal{S}_j)$ be the solution of (\ref{eq: tkm}), then substituting $t$, $\mathcal{S}$ and $\bc_j(t)$ into the tilted weight denoted as $\hat{w}_i:=w_i(t,\mathcal{S}_j,\bc_j(t))$, we can obtain the tilted empirical mean and variance for each cluster as
\begin{align}
    \mathbb{E}_t\Bigl(\mathrm{f}(\mathcal{S}_j,\bc_j)\Bigr)\!&=\!\!\!\sum_{\bx_i\in\mathcal{S}_j}\!\!\hat{w}_i\cdot f(\bx_i,\bc_j)\notag\\
    &=\|\bc_j\|^2+\!\!\!\sum_{\bx_i\in\mathcal{S}_j}\hat{w}_i\|\bx_i\|^2-\bc_j^{\top}M_t\notag\\
    \var_t\Bigl(\mathrm{f}(\mathcal{S}_j,\bc_j)\Bigr)\!&=\!\mathbb{E}_t\Bigl(\bc_j^{\top}\bigl(-2\bx_i-M_t\bigr)\Bigr)^2=\bc_j^{\top} V_t \bc_j,\notag
\end{align}
where $M_t\!:=\!2\!\sum_{\bx_i\in\mathcal{S}_j}\hat{w}_i\cdot\bx_i$ and $V_t\!:=\!\sum_{\bx_i\in\mathcal{S}_j}\!\hat{w}_i\bigl(-2\bx_i-M_t\bigr)^{\top}\bigl(-2\bx_i-M_t\bigr)$ are constants. Then, by taking derivative of $\var_{\tau}\Bigl(\mathrm{f}\bigl(\mathcal{S}_j,\bc_j(t)\bigr)\Bigr)$ with respect to $t$, we have
\begin{align}\label{eq55}
    &\frac{\partial }{\partial t}\Bigl\{\var_{\tau}\Bigl(\mathrm{f}\bigl(\mathcal{S}_j,\bc_j(t)\bigl)\Bigr)\Bigl\}\notag\\
    =&\Bigl(\frac{\partial}{\partial t}\bc_j(t)\Bigr)^{\top}\cdot\nabla_{\bc_j}\Bigl\{\var_{\tau}\Bigl(\mathrm{f}\bigl(\mathcal{S}_j,\bc_j(t)\bigl)\Bigr)\Bigr\}\notag\\
    =&2\Bigl(\frac{\partial}{\partial t}\bc_j(t)\Bigr)^{\top}V_{\tau} \bc_j(t).
\end{align}


Based on the optimal condition with $\bc_j$, we have
\begin{align}\label{eq: opti conditon}
    0&=\sum_{\bx_i\in\mathcal{S}_j} e^{t\|\bx_i-\bc_j(t)\|^2}(\bx_i-\bc_j(t)).
\end{align}

Divide both sides of \eqref{eq: opti conditon} by $-\frac{1}{2}\sum_{\bx_i\in\mathcal{S}_j} e^{t\|\bx_i-\bc_j(t)\|^2}$, and differentiate with respect to $t$ yields
\begin{align}
    &0=\frac{\partial}{\partial t}\Bigl\{\sum_{\bx_i\in\mathcal{S}_j}w_i(t,\mathcal{S}_j,\bc_j(t))\cdot2(\bc_j(t)-\bx_i)\Bigr\}\notag\\
    &\!\!\!\!=\frac{\partial}{\partial t}\Bigl\{2\bc_j(t) -  M(t,\mathcal{S}_j,\bc_j(t))\Bigr\}\notag\\
    &\!\!\!\!=\frac{\partial\bc_j(t)}{\partial t}\Bigl(2\!-\!\nabla_{\bc_j}M(t,\mathcal{S}_j,\bc_j(t))\!\Bigr)\!-\!\frac{\partial}{\partial \tau}M(\tau,\mathcal{S}_j,\bc_j(t))\Big|_{\tau=t} \!\!  \label{eq 47 follows from chain rule}\\
    &\!\!\!\!=\frac{\partial\bc_j(t)}{\partial t}\Bigl(2\!+\!tV(t,\mathcal{S}_j,\bc_j(t))\Bigr)\!+\!V(t,\mathcal{S}_j,\bc_j(t))\bc_j(t),\!\!\label{eq56}
\end{align}
where \eqref{eq 47 follows from chain rule} follows from the chain rule, and \eqref{eq56} follows from Lemma \ref{lemma:Partial derivatives}.
Then we can infer from (\ref{eq56}) that
\begin{align}\label{eq57}
    \frac{\partial\bc_j(t)}{\partial t} = -V(t,\mathcal{S}_j,\bc_j(t))\bc_j(t)\cdot\frac{1}{2+tV(t,\mathcal{S}_j,\bc_j(t))}.
\end{align}

Substituting (\ref{eq57}) into (\ref{eq55}), we obtain
\begin{align*}
    \frac{\partial }{\partial t}\Bigl\{\var_{\tau}\Bigl(\mathrm{f}\bigl(\mathcal{S}_j,\bc_j(t)\bigl)\Bigr)\Bigl\}&=2\Bigl(\frac{\partial}{\partial t}\bc_j(t)\Bigr)^{\top}V_{\tau} \bc_j(t)\notag\\
    &=\underbrace{-\frac{\bc_j(t)^{\top}V(t,\mathcal{S}_j,\bc_j(t))V_{\tau} \bc_j(t)}{2+tV(t,\mathcal{S}_j,\bc_j(t))}}_{<0},
\end{align*}
which completes the proof.
\end{proof}

\subsection{Proof of Theorem \ref{theorem: time complexity}}\label{sec: proof of theorem complexity}
\begin{proof}
When initializing the centroids with $k$-means++, the required number of multiplications is $O(nkd)$.
The number of multiplication needed for assignment and refinement are $O(nkd)$ and $O(nkdE)$, respectively.
When we set the number of iterations to $T$, we can obtain the multiplication required for \tkm is $O(nkdET)$.
\end{proof}

\subsection{Proof of Theorem \ref{theorem: monotonicity}}\label{sec: proof of theorem monotonicity}
\begin{proof}
When $k=1$, we obtain
\begin{align}
    \overline{\phi}(t,\mathcal{S},\bc)=\frac{1}{t}\log\frac{1}{n}\sum_{i=1}^{n}e^{tf(\bx_i,\bc)},
\end{align}
where $\mathcal{S}=X$ and $\bc=\tm(t,\mathcal{S})$ are the unique cluster and centroid.
We directly take the partial derivative of $\overline{\phi}(t,\mathcal{S},\bc)$ with respect to $t$, yielding: 
\begin{align}
    &\frac{\partial \Pi(t,\mathcal{S}_j,\bc_j)}{\partial t}\notag\\
    =&\frac{1}{t}\frac{\sum_{i=1}^n e^{t\|\bx_i-\bc\|^2}\|\bx_i-\bc\|^2}{\sum_{i=1}^n e^{t\|\bx_i-\bc\|^2}}-\frac{1}{t^2}\log\frac{1}{n}\sum_{i=1}^n e^{t \|\bx_i-\bc\|^2}\notag\\
    =&-\frac{1}{t}\bc_j^{\top}\sum_{i=1}^n 2w_i(t,\mathcal{S},\bc)\bx_i-\frac{1}{t^2}\log\frac{1}{n}\sum_{i=1}^n e^{-2t\bc^{\top}\bx_i}\label{eq: 51 follows from normal}\\
    =&-\frac{1}{t}\bc_j^{\top}M(t,\mathcal{S},\bc)\!-\!\frac{1}{t^2}\Gamma(t,\mathcal{S},\bc)=:g(t,\mathcal{S},\bc)\label{eq: g(t)},
\end{align}
where \eqref{eq: 51 follows from normal} follows from the fact that all data points are normalized, and \eqref{eq: g(t)} defines $g(t,\mathcal{S}_j,\bc_j)$. Next, we consider
\begin{align}\label{eq:32}
    \frac{\partial}{\partial t}\{t^2g(t,\mathcal{S},\bc)\}&=\frac{\partial}{\partial t}\Bigl\{-t\bc^{\top}M(t,\mathcal{S},\bc)\!-\!\Gamma(t,\mathcal{S},\bc)\Bigr\}\notag\\
    &=t\bc^{\top}V(t,\mathcal{S},\bc)\bc,
\end{align}
where (\ref{eq:32}) follows from (\ref{eq25}), (\ref{eq27}) and the chain rule. 
Given that $t\bc^{\top}V(t,\mathcal{S},\bc)\bc\geq0$ for any $t\geq 0$, therefore $t^2g(t,\mathcal{S},\bc)$ is a monotonically increasing function with $t$, and its minimum value is attained at $t=0$. When $t=0$, we have
\begin{align}\label{eq g0}
    g(0,\mathcal{S},\bc)&:=\lim_{t\rightarrow 0}-\frac{\Gamma(t,\mathcal{S},\bc)+t \bc_j^{\top} M (t,\mathcal{S},\bc)}{t^2}\notag,\\
    &=\frac{1}{2}\bc^{\top}V (0,\mathcal{S},\bc)\bc,
\end{align}
where (\ref{eq g0}) follows from (\ref{eq25}), (\ref{eq27}) and L'Hôpital's rule.
Then we obtain $t^2g(t,\mathcal{S},\bc)\geq 0$, and consequently infer $g(t,\mathcal{S},\bc)\geq 0$ for any $t\geq0$. In conjunction with Equation (\ref{eq: g(t)}), Theorem \ref{theorem: monotonicity} is implied.
\end{proof}


%




\ifCLASSOPTIONcaptionsoff
  \newpage
\fi

\bibliographystyle{abbrv}
\bibliography{sample.bib}

\vspace{-4em}
\begin{IEEEbiography}
	 [{\includegraphics[width=1in,height=1.22in,clip,keepaspectratio]{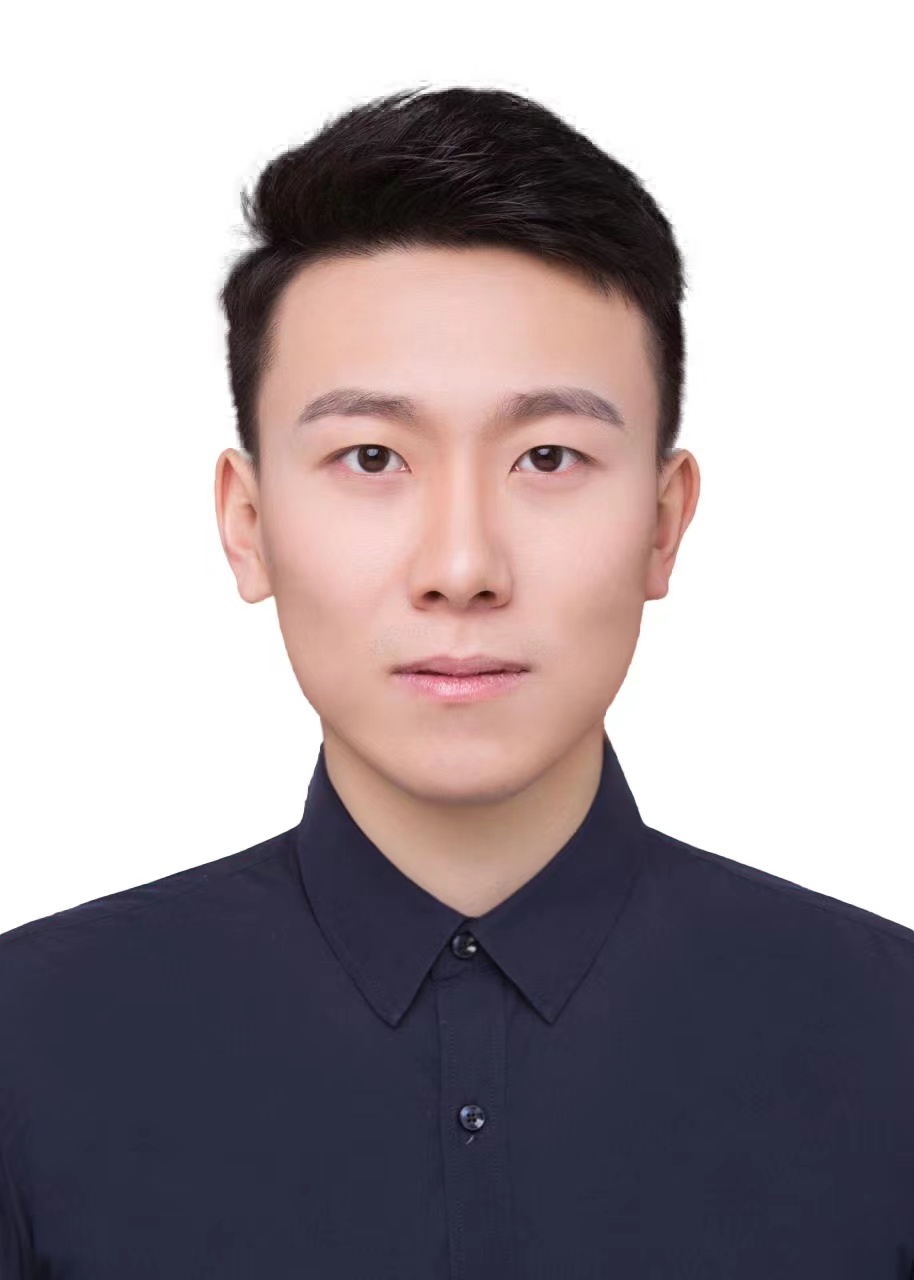}}]
	 {Shengkun Zhu} received the BE degree in electronic information engineering from Dalian University of Technology, China in 2018. He is currently working toward a Ph.D. degree in computer science and technology, School of Computer Science, Wuhan University, China. His research interests mainly include fair clustering, federated learning and nonconvex optimization. 
\end{IEEEbiography}
\begin{IEEEbiography}
	 [{\includegraphics[width=1in,height=1.22in,clip,keepaspectratio]{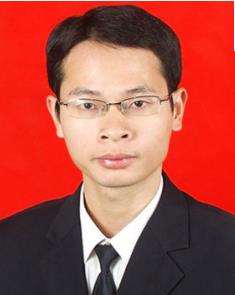}}]
	 {Jinshan Zeng} received the Ph.D. degree in mathematics from Xi’an Jiaotong University, Xi’an, China, in 2015. He is currently a Distinguished Professor with the School of Computer and Information Engineering, Jiangxi Normal University, Nanchang, China, and serves as the Director of the Department of Data Science and Big Data. He has authored more than 40 papers in high-impact journals and conferences such as IEEE TPAMI, JMLR, IEEE TSP, ICML, and AAAI. He has coauthored two papers with collaborators that received the International Consortium of Chinese Mathematicians (ICCM) Best Paper Award in 2018 and 2020). His current research interests include nonconvex optimization, machine learning (in particular deep learning), and remote sensing.
\end{IEEEbiography}

\begin{IEEEbiography}
	 [{\includegraphics[width=1in,height=1.20in,clip,keepaspectratio]{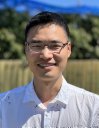}}]
	{Yuan Sun} is a Lecturer in Business Analytics and Artificial Intelligence at La Trobe University, Australia. He received his BSc in Applied Mathematics from Peking University, China, and his PhD in Computer Science from The University of Melbourne, Australia. His research interest is on artificial intelligence, machine learning, operations research, and evolutionary computation. He has contributed significantly to the emerging research area of leveraging machine learning for combinatorial optimisation. His research has been published in top-tier journals and conferences such as IEEE TPAMI, IEEE TEVC, EJOR, NeurIPS, ICLR, VLDB, ICDE, and AAAI. 
\end{IEEEbiography}

\vspace{-4em}

\begin{IEEEbiography}
	 [{\includegraphics[width=1in,height=1.22in,clip,keepaspectratio]{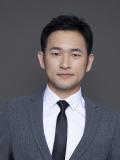}}]
	 {Sheng Wang} received the BE degree in information security, ME degree
	in computer technology from Nanjing University of Aeronautics and Astronautics,
	China in 2013 and 2016, and Ph.D. from RMIT University in 2019. He is a professor at the School of Computer Science, Wuhan University. His research interests mainly include mobile databases, multi-modal data management, and fairness-aware data analysis. He has published full research papers on top database and information systems venues as the first author, such as TKDE, SIGMOD, PVLDB, and ICDE.
\end{IEEEbiography}

\vspace{-4em}
\begin{IEEEbiography}
	 [{\includegraphics[width=1in,height=1.21in,clip,keepaspectratio]{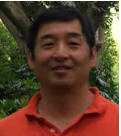}}]
	{Xiaodong Li} (Fellow, IEEE) received the B.Sc. degree from Xidian University, Xi’an, China, in 1988, and the Ph.D. degree in information science from the University of Otago, Dunedin, New Zealand, in 1998. He is a Professor with the School of Science (Computer Science and Software Engineering), RMIT University, Melbourne, VIC, Australia. His research interests include machine learning, evolutionary computation, neural networks, data analytics, multiobjective optimization, multimodal optimization, and swarm intelligence. Prof. Li is the recipient of the 2013 ACM SIGEVO Impact Award and the 2017 IEEE CIS IEEE TRANSACTIONS ON EVOLUTIONARY COMPUTATION Outstanding Paper Award. He serves as an Associate Editor for the IEEE TRANSACTIONS ON EVOLUTIONARY COMPUTATION, Swarm Intelligence (Springer), and International Journal of Swarm Intelligence Research. He is a Founding Member of IEEE CIS Task Force on Swarm Intelligence, the Vice-Chair of IEEE Task Force on Multimodal Optimization, and the Former Chair of IEEE CIS Task Force on Large Scale Global Optimization.
\end{IEEEbiography}


\vspace{-4em}
\begin{IEEEbiography}
	 [{\includegraphics[width=1in,height=1.21in,clip,keepaspectratio]{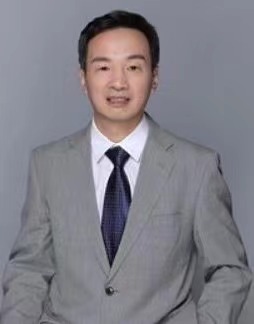}}]
	{Zhiyong Peng} received the BSc degree from
Wuhan University, in 1985, the MEng degree from the Changsha Institute of Technology of China, in 1988, and the PhD degree from the Kyoto University of Japan, in 1995. He is a professor of computer school, the Wuhan University of China. 
He worked as a researcher in Advanced Software Technology \& Mechatronics Research Institute of Kyoto from 1995 to 1997 and a member of technical staff in Hewlett-Packard Laboratories Japan from 1997 to 2000.
His research interests include complex data management, web data management, and trusted data management. 
He is a member of IEEE Computer Society, ACM SIGMOD and vice director of Database Society of Chinese Computer Federation. He was general co-chair of WAIM 2011, DASFAA 2013 and PC Co-chair of DASFAA 2012, WISE 2006, and CIT 2004.
\end{IEEEbiography}





\end{document}